\documentclass{article}

\usepackage{microtype}
\usepackage{graphicx}

\usepackage{booktabs} 

\usepackage{hyperref}

\usepackage[accepted]{icml2019}

\icmltitlerunning{Breaking the gridlock in MoE: Consistent and Efficient Algorithms}

\usepackage[T1]{fontenc}
\usepackage{amsmath,amssymb,amsthm}
\usepackage{prettyref,xspace,graphicx}
\usepackage{tikz}
\usepackage{authblk}
\usepackage{bbding}
\newcommand{\xmark}{\text{\ding{55}}}
\usepackage{pifont}

\usetikzlibrary{decorations.pathreplacing}

\usepackage{subcaption}

\usepackage{epstopdf}
\usepackage{dsfont}

\newrefformat{cond}{Condition~\ref{#1}}
\newrefformat{eq}{(\ref{#1})}
\newrefformat{thm}{Theorem~\ref{#1}}
\newrefformat{th}{Theorem~\ref{#1}}
\newrefformat{chap}{Chapter~\ref{#1}}
\newrefformat{sec}{Section~\ref{#1}}
\newrefformat{algo}{Algorithm~\ref{#1}}
\newrefformat{fig}{Figure~\ref{#1}}
\newrefformat{tab}{Table~\ref{#1}}
\newrefformat{rmk}{Remark~\ref{#1}}
\newrefformat{clm}{Claim~\ref{#1}}
\newrefformat{def}{Definition~\ref{#1}}
\newrefformat{cor}{Corollary~\ref{#1}}
\newrefformat{lmm}{Lemma~\ref{#1}}
\newrefformat{prop}{Proposition~\ref{#1}}
\newrefformat{pr}{Proposition~\ref{#1}}
\newrefformat{app}{Appendix~\ref{#1}}
\newrefformat{prob}{Problem~\ref{#1}}
\newrefformat{ques}{Question~\ref{#1}}
\newrefformat{note}{Note~\ref{#1}}
\newrefformat{assump}{Assumption~\ref{#1}}
\newrefformat{issue}{Issue ~\ref{#1}}
\newrefformat{fix}{Fix ~\ref{#1}}

\newcommand{\reals}{\mathbb{R}}

\newcommand{\Expect}{\mathbb{E}}

\newcommand{\prob}[1]{\mathbb{P}\left[#1\right]}

\newcommand{\inner}[2]{\langle {#1}, {#2}  \rangle}

\newcommand{\calB}{\mathcal{B}}

\newcommand{\calE}{\mathcal{E}}

\newcommand{\calN}{\mathcal{N}}

\newcommand{\calP}{\mathcal{P}}

\newcommand{\calS}{\mathcal{S}}
\newcommand{\calT}{\mathcal{T}}

\newcommand{\vect}[1]{\boldsymbol{#1}}

\newcommand{\ba}{\vect{a}}
\newcommand{\bb}{\vect{b}}

\newcommand{\be}{\vect{e}}

\newcommand{\bv}{\vect{v}}
\newcommand{\bw}{\vect{w}}
\newcommand{\bx}{\vect{x}}
\newcommand{\by}{\vect{y}}
\newcommand{\bz}{\vect{z}}

\newcommand{\bA}{\vect{A}}

\newcommand{\bW}{\vect{W}}

\newcommand{\norm}[1]{\left \|#1\right \|}

\newcommand{\add}[3]{\sum_{#1=#2}^{#3}}

\newcommand{\sym}{\mathrm{sym}}
\newcommand{\define}{\triangleq}

\newcommand{\pth}[1]{\left( #1 \right)}
\newcommand{\qth}[1]{\left[ #1 \right]}
\newcommand{\sth}[1]{\left\{ #1 \right\}}

\newcommand{\ie}{i.e.\xspace}
\newcommand{\iid}{i.i.d.\xspace}

\DeclareMathOperator*{\argmin}{\arg\!\min}

\newtheorem{theorem}{Theorem}
\newtheorem{lemma}{Lemma}
\newtheorem{assumption}{Assumption}

\newtheorem{remark}{Remark}

\newtheorem{condition}{Condition}

\newtheorem{definition}{Definition}
\newtheorem{example}{Example}

\begin{document}

\twocolumn[
\icmltitle{Breaking the gridlock in Mixture-of-Experts:\\ Consistent and Efficient Algorithms}




\begin{icmlauthorlist}
\icmlauthor{Ashok Vardhan Makkuva}{uiuc_ece}
\icmlauthor{Sewoong Oh}{uw_cs}
\icmlauthor{Sreeram Kannan}{uw}
\icmlauthor{Pramod Viswanath}{uiuc_ece}
\end{icmlauthorlist}

\icmlaffiliation{uiuc_ece}{Department of Electrical and Computer  Engineering, Coordinated Science Laboratory, University of Illinois at Urbana-Champaign, IL, USA}

\icmlaffiliation{uw_cs}{Allen School of Computer Science \& Engineering, University of Washington, Seattle, USA}

\icmlaffiliation{uw}{Department of Electrical Engineering, University of Washington, Seattle, USA}

%
\icmlcorrespondingauthor{Ashok Vardhan Makkuva}{makkuva2@illinois.edu}

\icmlkeywords{Machine Learning, ICML}

\vskip 0.3in
]



\printAffiliationsAndNotice{}  

\begin{abstract}
Mixture-of-Experts (MoE) is a widely popular model for ensemble learning and is a basic building block of highly successful modern neural networks as well as a component in Gated Recurrent Units (GRU) and Attention networks. However, present algorithms for learning MoE, including the EM algorithm and gradient descent, are known to get stuck in local optima. From a theoretical viewpoint, finding an efficient and provably consistent algorithm to learn the parameters remains a long standing open problem for more than two decades. In this paper, we introduce the first algorithm that learns the true parameters of a MoE model for a wide class of non-linearities with global consistency guarantees. While existing algorithms jointly or iteratively estimate the expert parameters and the gating parameters in the MoE, we propose a novel algorithm that breaks the deadlock and can directly estimate the expert parameters by sensing its echo in a  carefully designed cross-moment tensor between the inputs and the output. Once the experts are known, the recovery of gating parameters still requires an EM algorithm; however, we show that the EM algorithm for this simplified problem, unlike the joint EM algorithm, converges to the true parameters. We empirically validate our algorithm on both the synthetic and real data sets in a variety of settings, and show superior performance to standard baselines.

\end{abstract}

\section{Introduction}

In this paper, we study a popular gated neural network architecture known as Mixture-of-Experts (MoE). MoE is a basic building block of highly successful modern neural networks like Gated Recurrent Units (GRU) and Attention networks. A key interesting feature of MoE is the presence of a gating mechanism that allows for specialization of experts in their respective domains. MoE allows for the underlying expert models to be simple while allowing to capture complex non-linear relations between the data. Ever since their inception more than two decades ago \citep{JacJor}, they have been a subject of great research interest \citep{gpmoe, svmmoe,hmegp,  theis,le2016lstm, gross2017hard,sun2017human, wang2018deep} across multiple domains such as computer vision, natural language processing, speech recognition, finance, and forecasting. 

The basic MoE model is the following: let $\bx \in \reals^d$ be the input feature vector and $y \in \reals$ be the corresponding label. Then the discriminative model $P_{y|\bx}$ for the $k$-mixture of experts ($k$-MoE) in the regression setting is:
\begin{align}
P_{y|\bx} &= \sum_{i=1}^k P_{i|\bx} P_{y|\bx,i} \nonumber \\
&=\sum_{i=1}^k \frac{e^{\inner{\bw_i^\ast}{\bx}}}{\sum_{j}e^{\inner{\bw_j^\ast}{\bx}}}  \calN(y|g(\inner{\ba_i^\ast}{\bx}),\sigma^2).
 \label{eq:kmoe}
\end{align}
\prettyref{fig:kmoe} details the architecture for $k$-MoE. 

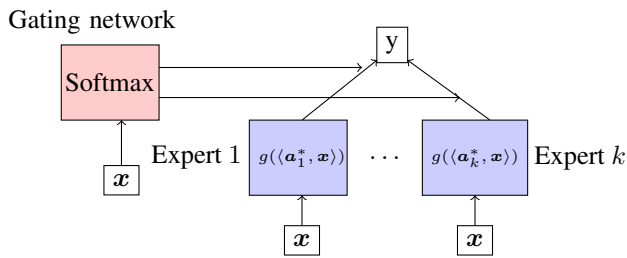
\begin{figure}[H]
\centering
\begin{tikzpicture}[domain=0:1.5,xscale=1,yscale=1]

\filldraw[fill=blue!20!white, draw=black] (-2.3,0) -- (-1.0,0) -- (-1.0,1) -- (-2.3,1) -- cycle;
\node at (-1.6,0.5) {$\tiny g(\langle \ba_1^*, \bx \rangle )$};
\node at (-3.0,0.5) {Expert $1$};

\node at (-0.5,0.5) {$\ldots$};

\filldraw[fill=blue!20!white, draw=black] (0.0,0) -- (1.4,0) -- (1.4,1) -- (0.0,1) -- cycle;
\node at (0.7,0.5) {\tiny $g( \langle \ba_k^*, \bx\rangle )$};
\node at (2.1,0.5) {Expert $k$};

\node[draw] at (0.7,-0.6) {$\bx$};
\draw[->] (0.7,-0.4)  -- (0.7,0);

\node[draw] at (-1.6,-0.6) {$\bx$};
\draw[->] (-1.6,-0.4)  -- (-1.6,0);

\draw[->] (0.9,1)  -- (-0.2,1.8);
\draw[->] (-1.6,1)  -- (-0.6,1.8);
\node[draw] at (-0.4,2) {y};

\filldraw[fill=red!20!white, draw=black] (-3.5,1.0) -- (-4.8,1.0) -- (-4.8,2.0) -- (-3.5,2.0) -- cycle;
\node[text width=2cm] at (-3.75,1.5) {Softmax};
\node[text width=3cm] at (-4.0,2.3) {Gating network};

\draw[->] (-4.0,0.4)  -- (-4.0,1.0);
\node[draw] at (-4.0,0.2) {$\bx$};

\draw[->](-3.5,1.3) -- (0.5,1.3);
\draw[->] (-3.5,1.7) -- (-0.8,1.7);

\end{tikzpicture}
\caption{Architecture for $k$-MoE}
\label{fig:kmoe}
\end{figure}

The interpretation behind \prettyref{eq:kmoe} is that for each input $\bx$, the gating network chooses an expert based on the outcome of a multinomial random variable $z \in [k]$, whose probability depends on $\bx$ in a parametric way, \ie $z|\bx \sim \mathrm{softmax}(\inner{\bw_1^\ast}{\bx}, \ldots, \inner{\bw_k^\ast}{\bx})$. The chosen expert then generates the output $y$ from a Gaussian distribution centred at a non-linear activation of $\bx$, \ie $g(\inner{\ba_z^\ast}{ \bx})$, with variance $\sigma^2$. 
We want to learn the  expert parameters $\ba_i^\ast \in \reals^d$ (also referred to as the regressors) and 
the gating parameters $\bw_i^\ast \in \reals^d$, 
assuming we know the non-linear activation $g:\reals \to \reals$.

This problem of learning MoE 
has been a long standing open problem for more than two decades, 
even though it is a fundamental building block of several state-of-the-art gated neural network architectures.  
Gated neural networks such as GRUs and Sparsely-gated-MoEs 
 have been widely successful in challenging tasks like machine translation \citep{empiricalgru, SMMD+17, vaswani2017attention}. 
Parameters are typically learnt through (stochastic) gradient descent on a non-convex  loss function. 
However, these methods do not possess any theoretical guarantees, even for the simplest gated neural network, which is the MoE. 

On the other hand, existing guarantees for simpler models without gating units do not extend to MoEs. 
Consider the mixture of generalized linear models (M-GLMs) \citep{SedghiA14a,SIM14,YCS16,zhong16glm}, 
which is a strict simplification of 
the $k$-MoE model in \prettyref{eq:kmoe},  
where  $\bw_i^\ast=0$ for all $i\in\{1,\ldots, k\}$. 
 The learning in  M-GLMs is usually done through a combination of spectral methods and greedy methods such as EM. A major limitation of these  methods is that they rely  critically on the fact that the mixing probability is a constant and hence they do not generalize to MoEs (see \prettyref{sec:algorithmsmoe}). In addition, the EM algorithm, which is the workhorse for learning in parametric mixture models, is prone to bad local minima \citep{SedghiA14a,bala17,zhong16glm} (we independently verify this for MoEs in \prettyref{sec:experiments}). 
These theoretical shortcomings and practical relevance of the MoE models lead to the  following fundamental question:

Can we find an efficient and a consistent algorithm (with global initializations) that recovers the true parameters of the model with theoretical guarantees?

In this paper, we address this question precisely and make the following contributions:


\textbf{1) First theoretical guarantees:} We provide the first (poly-time) efficient algorithm  that recovers the true parameters of a MoE model with global initializations 
 (\prettyref{thm:twoexpertspopulation} and \prettyref{thm:twoexpertsEM}). We allow for a wide class of non-linearities which includes the popular choices of identity, sigmoid, and ReLU. To the best of our knowledge, ours is the first work to give global convergence guarantees for MoE.

\textbf{2) Algorithmic innovations:} 
Existing algorithms jointly or iteratively estimate the expert parameters and the gating paramters in the MoE and can get stuck in local minima. In this paper, we propose a novel algorithm that breaks the gridlock and can directly estimate the expert parameters by sensing its echo in a  cross-moment tensor between the inputs and the output (\prettyref{algo:regressrecov} and \prettyref{algo:gatingrecov}). Once the experts are known, the recovery of gating parameters still requires an EM algorithm; however, we show that the EM algorithm for this simplified problem, unlike the joint EM algorithm, converges to the true parameters. 
 The proofs of global convergence of EM as well as the design of the cross-moment tensor are of independent mathematical interest.

\textbf{3) Novel transformations:} In this paper, we introduce the novel notion of ``Cubic and Quadratic Transform (CQT)". These are polynomial transformations on the output labels tailored to specific non-linear activation functions and the noise variance. The key utility of these transforms is to equip MoEs with a supersymmetric tensor structure in a principled way (\prettyref{thm:twoexpertspopulation}). 
\begin{figure*}[h]
\centering
\begin{tikzpicture}

\filldraw[fill=red!20!white, draw=black] (-2.5,1.0) -- (-4.8,1.0) -- (-4.8,2.0) -- (-2.5,2.0) -- cycle;
\node[text width=2cm] at (-3.2 ,1.7) {Tensor};
\node[text width=2cm] at (-3.7 ,1.3) {decomposition};
\draw[->] (-2.5,1.5) -- (-1.5,1.5) ;
\node at (0.4,1.5) {Regressors $\{\hat{\ba}_1,\ldots,\hat{\ba}_k \}$};

\filldraw[fill=red!20!white, draw=black] (-2.5,-1.0) -- (-4.8,-1.0) -- (-4.8,-2.0) -- (-2.5,-2.0) -- cycle;
\node[text width=2cm] at (-3.0,-1.3) {EM};
\node[text width=2cm] at (-3.5,-1.7) {algorithm};
\draw[->] (-2.5,-1.5) -- (-1.5,-1.5) ;
\node at (1.1,-1.5) {Gating parameters $\{\hat{\bw}_1,\ldots,\hat{\bw}_{k-1} \}$};

\draw[->] (0.8,1.2) -- (-3.5,-1.0);

\filldraw[fill=blue!20!white, draw=black] (-6.0,1.2) -- (-9.0,1.2) -- (-9.0,0.0) -- (-6.0,0.0) -- cycle;
\node at (-7.5,0.8) {Cubic \& Quadratic };
\node at (-7.5,0.4) {Transform };
\draw[->] (-6.0,0.6) -- (-4.8,1.4) ;

\filldraw[fill=blue!20!white, draw=black] (-6.0,3.2) -- (-9.0,3.2) -- (-9.0,2.0) -- (-6.0,2.0) -- cycle;
\node at (-7.5,2.6) {Score function};
\draw[->] (-6.0,2.6) -- (-4.8,1.6) ;

\filldraw[fill=white, draw=black] (-10.0,3.2) -- (-11.0,3.2) -- (-11.0,0.0) -- (-10.0,0.0) -- cycle;
\node at (-10.5,3.5) {Samples};
\node at (-10.5,2.6) {$\bx$};
\draw[->] (-10.0,2.6) -- (-9.0,2.6);

\node at (-12.7,1.6) {\bf {Algorithm $1$}};
\draw[->] (-11.6,1.6) -- (-11.0,1.6);

\node at (-12.7,-0.8) {\bf { Algorithm $2$} };
\draw[->] (-11.5,-0.8) -- (-10.5,-0.8);


\node at (-10.5,0.6) {$y$};
\draw[->] (-10.0,0.6) -- (-9.0,0.6);

\draw (-10.5,0.0) -- (-10.5,-1.5);
\draw[->] (-10.5,-1.5) -- (-4.8,-1.5);

%
%
%
%

\end{tikzpicture}
\caption{Algorithm to learn the MoE parameters. \textbf{Algorithm $1$}: First we take non-linear transformations on the samples $(\bx_i,y_i)$ to compute the tensors $\calT_2,\calT_3$. Spectral decomposition on $\calT_2,\calT_3$ recovers the regressors. \textbf{Algorithm $2$}: EM uses the learnt regressors and samples to learn the gating parameters with random initializations}
\label{fig:scheme}
\end{figure*}
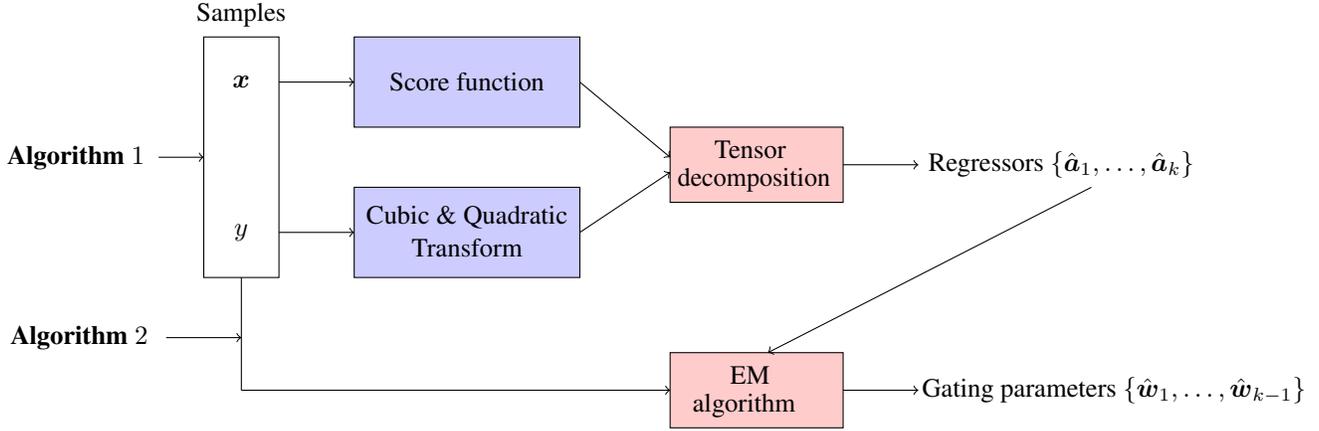

\textbf{Related work.} While there is a huge literature on MoEs (\cite{YWG12,MaEb14} are detailed surveys), there are relatively few works on its learning guarantees. \cite{emmoeanalysis} is the first work to analyze the local convergence of joint-EM for both the gating and the expert parameters. As noted earlier, however, EM is prone to bad local minima. In contrast, our algorithms have \emph{global convergence} guarantees. It is important to note that even for the simpler problem of mixtures of Gaussians, it is known that EM gets stuck in local minima, whenever number of mixtures, $k$, is at least $3$ \citep{localem}, whereas we can handle $2k-1<d$ with global convergence.

The simplified versions of MoE, M-GLMs, are widely studied in the literature. The key techiques for parameter inference in M-GLMs include EM algorithm, spectral methods, convex relaxations, and their variants.  \cite{yi2014alternating, bala17} prove convergence of EM for $2$-mixtures of linear regressions; in contrast, we handle $k\geq 2$ mixtures for a wide class of non-linearities and provide global convergence. \cite{SedghiA14a} construct a $3^{\mathrm{rd}}$-order supersymmetric tensor containing the regressors as its rank-$1$ components. However, this approach fails to generalize for MoE. \cite{zhong16glm} use a similar tensor construction followed by EM to learn the parameters; however, they can only handle linear noiseless mixtures and no gating parameters. In contrast, our algorithms can handle non-linearities and the gating parameters. \cite{chen2014convex} use a convex objective to learn the regressors for a special setting of $2$-mixtures of linear regressions. Similar to earlier approaches, this relaxation too does not generalize to $k >2$. 

\textbf{Notation.} In this paper, we denote Euclidean vectors by bold face 
lowercase letters $\ba, \bb$, etc., and scalars by plain lowercase letters 
$y,z$, etc. We use $\calN(y|\mu,\sigma^2)$ either to denote the density or the distribution of a Gaussian random variable $y$ with mean $\mu$ and variance $\sigma^2$, depending on the context. $[d] \define \{1,\ldots,d\} $. $\mathrm{Perm}[d]$ denotes the set of all permutations on $[d]$. We use $\otimes$ to denote the tensor outer product of vectors in $\reals^d$. $\bx^{\otimes 3}$ denotes $\bx \otimes \bx \otimes \bx$, where $(\bx \otimes \bx \otimes \bx)_{ijk}=x_i x_j x_k $. $\sym(\bx \otimes \by \otimes \bz) $ denotes the symmetrized version of  $\bx \otimes \by \otimes \bz$, \ie $\sym(\bx \otimes \by \otimes \bz)_{ijk} =  \sum_{\sigma \in \mathrm{Perm}[d]} x_{\sigma(i)} y_{\sigma(j)} z_{\sigma(k)}  $.  $\be_i, i \in [d]$ denotes the standard basis vectors for $\reals^d$.
 Through out the paper, we assume that $\bw_k^\ast=0$, without loss of generality. 

\section{Algorithms}
\label{sec:algorithmsmoe}
In this section, we present our algorithms to learn the regression and gating parameters \emph{separately}. \prettyref{fig:scheme} summarizes our algorithm. First we take a moment to highlight the issues of the existing approaches.

For illustration purposes, we suppose that $k=2$ in \prettyref{eq:kmoe}. 
We assume without loss of generality that $\bw^*_k=\bw^*_2=0$ and denote $\bw_1^\ast=\bw^\ast$.  Thus the $2$-MoE model is given by $P_{y|\bx}$: 
\begin{align}
 \hspace{-0.3cm} \frac{e^{\inner{\bw^\ast}{\bx}} \, \calN(y|g(\inner{\ba_1^\ast}{\bx}),\sigma^2)}{1+e^{\inner{\bw^\ast}{\bx}}}
 + \frac{\calN(y|g(\inner{\ba_2^\ast}{\bx}),\sigma^2)}{1+e^{\inner{\bw^\ast}{\bx}}}
 \label{eq:2MoE}
\end{align}

\textbf{Issues with traditional tensor methods.} 
In the far simplified setting of the absence of the gating parameter, \ie $\bw^\ast=0 \in \reals^d$, we see that $2$-MoE reduces to $2$-uniform mixture of GLMs. In this case, for $\bx \sim \calN(0,I_d)$, the standard approach is to construct a $3^{\mathrm{rd}}$-order tensor $\calT$ by regressing the output $y$ on the score transformation $\calS_3(\bx) \triangleq \bx \otimes \bx \otimes \bx - \sum_{i \in [d]} \mathrm{sym}\pth{\bx \otimes \be_i \otimes \be_i }$, 
\ie
\begin{align}
\calT  \define \Expect[y \cdot \calS_3(\bx)] \nonumber &= \frac{1}{2}\Expect[g'''(\inner{\ba_1^\ast}{\bx})] \cdot (\ba_1^\ast)^{\otimes 3} \nonumber \\
& \hspace{-1em} + \frac{1}{2} \Expect[g'''(\inner{\ba_2^\ast}{\bx})] \cdot (\ba_2^\ast)^{\otimes 3}\;. 
\label{eq:tensorpass}
\end{align}
Here the second equality follows from the generalized Stein's lemma that $\Expect[f(\bx)\cdot \calS_3(\bx)]=\Expect[\nabla_{\bx}^{(3)} f(\bx)  ] $ under some regularity conditions on $f:\reals^d \mapsto \reals$ (see \prettyref{lmm:highstein} in \prettyref{app:reviewtensor}). Then the regressors can be learned through spectral decomposition on $\calT$, where the uniqueness of decomposition follows from \cite{kruskal1977three}. If we apply a similar technique for $2$-MoE in \prettyref{eq:2MoE}, we obtain that
\begin{align}
\Expect[y \cdot \calS_3(\bx)] &= \sum_{i=1,2} \alpha_i  (\ba_i^\ast)^{\otimes 3}+\beta_i \, \sym(\ba_i^\ast \otimes \ba_i^\ast \otimes \bw^\ast)\nonumber \\
& \hspace{-1em} +\gamma_i \, \sym(\ba_i^\ast \otimes \bw^\ast \otimes \bw^\ast) + \delta (\bw^\ast)^{\otimes 3}, 
\label{eq:tensorfail}
\end{align} 
where $\alpha_i,\beta_i,\gamma_i, \delta$ are some scalar constants depending on the parameters $\ba_1^\ast, \ba_2^\ast, \bw^\ast$ and $g$ (see \prettyref{app:regressorlinear} for the proof). Thus \prettyref{eq:tensorfail} reveals that traditional spectral methods do not yield a supersymmetric tensor of the desired parameters for MoEs. In fact, \prettyref{eq:tensorfail} contains all the $3^{\mathrm{rd}}$-order rank-$1$ terms formed by $\ba_1^\ast,\ba_2^\ast$ and $\bw^\ast$. Hence we cannot recover these parameters uniquely. Note that the inherent coupling between the regressors $\ba_1^\ast,\ba_2^\ast$ and the gating parameter $\bw^\ast$ in \prettyref{eq:2MoE} manifests as a cross tensor in \prettyref{eq:tensorfail}. This coupling serves as a key limitation for the traditional methods which critically rely on the fact that the mixing probability $p=\frac{1}{2}$ in \prettyref{eq:tensorfail} is a constant. In fact, we recover \prettyref{eq:tensorpass} by letting $\bw^\ast=0$ in \prettyref{eq:tensorfail}.

\textbf{Issues with EM algorithm.} EM algorithm is the workhorse for parameter learning in both the $k$-MoE and HME models \citep{JJ94}. However, it is well known that EM is prone to spurious minima and existing theoretical results only establish local convergence for the regressors and the gating parameters. Indeed, our numerical experiments in \prettyref{sec:compjoint-EM} verify this fact. \prettyref{fig:k_3} and \prettyref{fig:k_4} highlight that joint-EM often gets stuck in bad local minima.

%

\subsection{The proposed algorithm for learning MoE}
\label{sec:algoMoE}

In order to tackle these challenges, we take a different route and propose to estimate the regressors and gating parameters \emph{separately}. 
To gain intuition about our approach, let us consider 2-MoE model in \prettyref{eq:2MoE} with $\sigma=0$ and linear $g$. Then we  have that $y$ either equals $\inner{\ba_1^\ast}{\bx}$ with probability $\sigma(\inner{\bw^\ast}{\bx})$ or equals $\inner{\ba_2^\ast}{\bx}$ with probability $1-\sigma(\inner{\bw^\ast}{\bx})$, where $\sigma(\cdot)$ is the sigmoid function. If we exactly know $\bw^\ast$, we can recover $\ba_1^\ast$ and $\ba_2^\ast$ by solving a simple linear regression problem since we can recover the true latent variable $\bz \in \{1,2\}$ with high probability. Similarly, if we know $\ba_1^\ast$ and $\ba_2^\ast$, it is easy to see that we can recover $\bw^\ast$ by solving a binary linear classification problem. Thus knowing either the regressors or the gating parameters makes the estimation of other parameters easier. However, how do we first obtain one set of parameters without any knowledge about the other?

Our approach precisely addresses this question and breaks the \emph{grid lock}. We show that we can extract the regressors $\ba_1^\ast$ and $\ba_2^\ast$ without knowing $\bw^\ast$ at all, just using the samples. 
Although we explain our approach with two mixtures, 
all claims are made precise for general $k$ in Theorems \ref{thm:twoexpertspopulation} and \ref{thm:twoexpertsEM}, 
and the algorithms are written for general $k$ as well in Algorithms \ref{algo:regressrecov} and \ref{algo:gatingrecov}. 

\subsubsection*{Step 1: Estimation of regressors}
To learn the regressors, we first pre-process $\bx \sim \calN(0,I_d)$ using the score transformations $\calS_3$ and $\calS_2$, \ie
\begin{align}
\calS_3(\bx) & \triangleq \bx \otimes \bx \otimes \bx - \sum_{i \in [d]} \mathrm{sym}\pth{\bx \otimes \be_i \otimes \be_i }, \\
 \calS_2(\bx) & \triangleq  \bx \otimes \bx - I.
\label{eq:thirdscore}
\end{align}
These score functions can be viewed as higher-order feature extractors from the inputs. As we have seen in \prettyref{eq:tensorpass}, these transformations suffice to learn the parameters in M-GLMs. However this approach fails in the context of MoE, as highlighted in \prettyref{eq:tensorfail}. Can we still construct a supersymmetric tensor for MoE?

To answer this question in a principled way, we introduce the notion of ``Cubic and Quadratic Transform (CQT)" for the labels, \ie
\begin{align*}
\calP_3(y) \triangleq y^3+\alpha y^2+ \beta y, \quad \calP_2(y) \triangleq y^2+\gamma y.
\end{align*} 
The coefficients $(\alpha,\beta,\gamma)$ in these polynomial transforms are obtained by solving a linear system of equations (see \prettyref{app:nonlinear}). For the special case of $g=$linear, we obtain $\calP_3(y)=y^3-3(1+\sigma^2)\,y$ and $\calP_2(y)=y^2$. These special transformations are specific to the choice of non-linearity $g$ and the noise variance $\sigma$. The key intuition behind the design of these transforms is that we can nullify the cross moments and obtain supersymmetric tensor in \prettyref{eq:tensorpass} if we regress $\calP_3(y)$ instead of $y$, for properly chosen constants $\alpha$ and $\beta$. This is made mathematically precise in \prettyref{thm:twoexpertspopulation}. A similar argument holds for $\calP_2(y)$ too.
In addition, the choice of these polynomials is unique in the sense that any other polynomial transformations fail to yield the desired tensor structure. Using these transforms, we construct two special tensors $\hat{\calT_3} \in (\reals^{d})^{\otimes 3}$ and $\hat{\calT_2} \in (\reals^{d})^{\otimes 2} $. Later we use the robust tensor power method \citep{AFHK+12} on these tensors to learn the regressors. \prettyref{algo:regressrecov} details our learning procedure. \prettyref{thm:twoexpertspopulation} establishes the theoretical justification for our algorithm. 

\begin{algorithm}
\caption{Learning the regressors}\label{algo:regressrecov}
\begin{algorithmic}[1]
\STATE {\bfseries Input:} Samples $(\bx_i,y_i), i \in [n]$ \STATE Compute $\hat{\calT_3}=\frac{1}{n}\sum_{i=1}^n \calP_3(y_i) \cdot \calS_3(\bx_i)$ and $\hat{\calT_2}=\frac{1}{n}\sum_{i=1}^n \calP_2(y_i) \cdot \calS_2(\bx_i)$ 
\STATE $\hat{\ba}_1,\ldots,\hat{\ba}_k=$ Rank-$k$ tensor decomposition on $\hat{\calT_3}$ using $\hat{\calT_2}$
\end{algorithmic}
\end{algorithm}

\subsubsection*{Step 2: Estimation of gating parameters}
To gain intuition for estimating the gating parameters, let $g=\mathrm{linear}$ in \prettyref{eq:2MoE} for simplicity. Moreover, assume that we know both $\ba_1^\ast$ and $\ba_2^\ast$.  Then taking conditional expectation on $y$, we obtain from \prettyref{eq:2MoE} that
\begin{align}
\Expect [ y | \bx] &= f(\inner{\bw^\ast}{\bx}) \cdot \inner{\ba_1^\ast}{\bx} + (1-f(\inner{\bw^\ast}{\bx})) \cdot \inner{\ba_2^\ast}{\bx}, \nonumber\\
&=\inner{\ba_2^\ast}{\bx} + f(\inner{\bw^\ast}{\bx}) \cdot \inner{\ba_1^\ast-\ba_2^\ast}{\bx},
\end{align}
where $f$ is the sigmoid function. 
Thus,
\begin{align*}
\Expect \qth{ \frac{  y  - \inner{\ba_2^\ast}{\bx}}{\inner{\ba_1^\ast-\ba_2^\ast}{\bx}} | \bx } = \frac{ \Expect [ y | \bx] - \inner{\ba_2^\ast}{\bx}}{\inner{\ba_1^\ast -\ba_2^\ast }{\bx}} &= f(\inner{\bw^\ast}{\bx}).
\end{align*}
Note that since $\bx$ is Gaussian, $\inner{\ba_1^\ast- \ba_2^\ast}{\bx}$ is non-zero with probability $1$. Hence, to recover $\bw^\ast$, in view of Stein's lemma, we may write
\begin{align*}
\Expect \qth{ \pth{\frac{  y  - \inner{\ba_2^\ast}{\bx}}{\inner{\ba_1^\ast-\ba_2^\ast}{\bx}} } \cdot \bx  } &\stackrel{\xmark} =  \Expect\qth{ f(\inner{\bw^\ast}{\bx}) \cdot \bx } \\
&= \Expect\qth{f'(\inner{\bw^\ast}{\bx})} \cdot \bw^\ast \\
&= \Expect_{Z \sim \calN(0,1)} f'(\norm{\bw^\ast} Z) \cdot \bw^\ast\\
& \propto \bw^\ast.
\end{align*}
However, it turns out that the above chain of equalities does not hold. Surprisingly, the first equality, which essentially is the law of iterated expectations, is not valid in this case as $\frac{  y  - \inner{\ba_2^\ast}{\bx}}{\inner{\ba_1^\ast-\ba_2^\ast}{\bx}}$ is not integrable since it is a mixture of two Cauchy distributions, as proved in \prettyref{app:gatingspectral}. Thus the above analysis highlights the difficulty of learning the gating parameters even in the simplest setting of two linear mixtures. Can we still learn $\bw^\ast$ using method of moments (MoM)? In \prettyref{thm:MoM_gating}, we precisely address this question and show that we can still provably recover the gating parameters using MoM, by designing clever transformations on the data to infer the parameters of a Cauchy mixture distribution. 

While \prettyref{thm:MoM_gating} highlights that gating parameters can be learnt using the method of moments for $2$-MoE, we still need a principled approach to learn these parameters for a more generic setting of $k$-MoE. Recall that the traditional joint-EM algorithm randomly initializes both the regressors and the gating parameters and updates them iteratively. \prettyref{fig:k_3} and \prettyref{fig:k_4} highlight that this procedure is prone to spurious minima. Can we still learn the gating parameters with \emph{global initializations}? To address this question, we utilize the regressors learnt from \prettyref{algo:regressrecov}. In particular, we use EM algorithm to update \emph{only} the gating parameters, while fixing the regressors $\hat{\ba}_1,\ldots,\hat{\ba}_k$. We show in \prettyref{thm:twoexpertsEM} that, with \emph{global/random} initializations, this variant of EM algorithm learns the true parameters. To the best of our knowledge, this is the first global convergence result for EM for $k>2$ mixtures. This motivates the following algorithm ($\varepsilon>0$ is some error tolerance):

\begin{algorithm}
\caption{Learning the gating parameter}\label{algo:gatingrecov}
\begin{algorithmic}[1]
\STATE {\bfseries Input:} Samples $(\bx_i,y_i), i \in [n]$ and regressors $\hat{\ba_1},\ldots,\hat{\ba_k}$ from \prettyref{algo:regressrecov}
\STATE $t \gets 0$
\STATE Initialize $\bw_0$ uniformly randomly in its domain $\Omega$
\WHILE{(Estimation error $< \varepsilon$ )}
\STATE Compute the posterior $p_{\bw_t}^{(i)}$ according to \prettyref{eq:computep} for each $j \in [k]$ and $i \in [n]$
\STATE Compute $Q(\bw|\bw_t)$ according to \prettyref{eq:Qdef} using empirical expectation
\STATE Set $\bw_{t+1}= \mathrm{argmax}_{\bw \in \Omega }Q(\bw|\bw_t) $
\STATE $t \gets t+1$
\STATE Estimation error $=\norm{\bw_t-\bw_{t-1}}$
\ENDWHILE

\end{algorithmic}
\end{algorithm}

\section{Theoretical analysis}
\label{sec:guaranteestwomoe}
In this section, we provide the theoretical guarantees for our algorithms in the population setting. We first formally state our assumptions and justify the rationale behind them:
\begin{enumerate}
\item $\bx$ follows standard Gaussian distribution, \ie $\bx \sim \calN(0,I_d)$.
\item $\norm{\ba_i^\ast}_2 =1$ for $i \in [k]$ and $\norm{\bw_i^\ast}_2 \leq R$ for $i \in [k-1]$, with some $R>0$.
\item $\ba_i^\ast, i \in [k]$ are linearly independent and $\bw_i^\ast$ is orthogonal to $\mathrm{span}\{\ba_1^\ast,\ldots,\ba_k^\ast \}$ for $i \in [k-1]$.
\item The non-linearity $g:\reals \to \reals$ is $(\alpha,\beta,\gamma)$-valid, which we define in \prettyref{app:nonlinear}. For example, this class includes $g=$linear, sigmoid and ReLU.
\end{enumerate}

\textbf{Remark.} We note that the Gaussianity of the input distribution and norm constraints on the parameters are standard assumptions in the learning of neural networks literature \citep{janzamin2015beating, li2017convergence,ge2017learning, 2017recovery, du2017gradient, safran2017spurious} and also that of M-GLMs \citep{SedghiA14a,YCS16,zhong16glm,bala17}. An interpretation behind Assumption $3$ is that if we think of $\bx$ as a
high-dimensional feature vector, distinct sub-features of $\bx$ are used to perform the two distinct tasks of classification (using $\bw_i^\ast$'s) and regression (using $\ba_i^\ast$'s). We note that we need the above assumptions only for the technical analysis. In \prettyref{sec:non-gaussian} and \prettyref{sec:non-orthogonal}, we empirically verify that our algorithms work well in practice even under the relaxation of these assumptions. Thus we believe that the assumptions are merely technical artifacts.

We are now ready to state our results. 

\begin{theorem}[Recovery of regression parameters]
\label{thm:twoexpertspopulation}
Let $(\bx,y)$ be generated according the true model \prettyref{eq:kmoe}. Under the above assumptions, we have that
\begin{align*}
\calT_2  \triangleq \Expect[\calP_2(y) \cdot \calS_2(\bx)] & = \add{i}{1}{k} c'_{g} \Expect[P_{i|\bx}] \cdot \ba_i^\ast \otimes \ba_i^\ast,\\
\calT_3  \triangleq \Expect[\calP_3(y) \cdot \calS_3(\bx)] & = \add{i}{1}{k} c_{g,\sigma}  \Expect[P_{i|\bx}] \cdot \ba_i^\ast \otimes \ba_i^\ast \otimes \ba_i^\ast,
\end{align*}
where $c'_{g}$ and $c_{g,\sigma}$ are two non-zero constants depending on $g$ and $\sigma$. Hence the regressors $\ba_i^\ast$'s can be learnt through tensor decomposition on $\calT_2$ and $\calT_3$. 
\end{theorem}

\begin{proof}(Sketch)
To highlight the central ideas behind the proof, first let $g=$linear. From \prettyref{eq:kmoe} we get that
\begin{align*}
\Expect[y|\bx] &= \sum_{i \in [k]} p_i^\ast(\bx) \inner{\ba_i^\ast}{\bx},
\end{align*}
where $p_i^\ast(\bx)\define P_{i|\bx}$ for $i \in [k]$. Taking the cross moment of $y$ with $\calS_3(\bx)$ and using \prettyref{lmm:highstein} we obtain that
\begin{align*}
\Expect[y \cdot \calS_3(\bx)]&= \sum_{i \in [k]}\Expect[p_i^\ast(\bx) \inner{\ba_i^\ast}{\bx} \cdot \calS_3(\bx)]\\
&=\sum_{i \in [k]}\Expect[\nabla_{\bx}^{(3)}(p_i^\ast(\bx) \inner{\ba_i^\ast}{\bx})].
\end{align*}
Notice that had $p_i^\ast(\bx)$ been a constant in the above equation, we would obtain a supersymmetric tensor easily as is the case with M-GLMs. However, $\Expect[\nabla_{\bx}^{(3)}(p_i^\ast(\bx) \inner{\ba_i^\ast}{\bx})]$ now contains all the third-order rank-$1$ terms involving the tensor product of $\bw_1^\ast,\ldots,\bw_{k-1}^\ast$ and $\ba_i^\ast$ for any fixed $i$. Our key insight is that this issue can be avoided if we cleverly transform $y$. In particular, we consider a cubic transformation $\calP_3(y)=y^3-3y(1+\sigma^2)$ and obtain that
\begin{align*}
\Expect[\calP_3(y)|\bx]=\sum_{i \in [k]} p_i^\ast(\bx) (\inner{\ba_i^\ast}{\bx}^3-3\inner{\ba_i^\ast}{\bx})
\end{align*}
Now it turns out that after using the orthogonality of $\bw_i^\ast$ and $\ba_i^\ast$, and the fact $\Expect[p(Z)]=\Expect[p'(Z)]=\Expect[p''(Z)]=0$ for $3^{\mathrm{rd}}$-Hermite polynomial $p(z)=z^3-3z$ and $Z \sim \calN(0,1)$, we can nullify the cross-moments between $\bw_i^\ast$'s and $\ba_i^\ast$'s to obtain that
\begin{align*}
\Expect[\calP_3(y)\cdot \calS_3(\bx)]=6 \sum_{i \in [k]}\Expect[p_i^\ast(\bx)]. (\ba_i^\ast)^{\otimes 3}.
\end{align*}
Similarly, we can show that $\Expect[\calP_2(y)\cdot \calS_2(\bx)]=2 \sum_{i \in [k]}\Expect[p_i^\ast(\bx)]. (\ba_i^\ast)^{\otimes 2}$. For a general non-linearity $g:\reals \to \reals$, we can similarly design cubic and quadratic polynomials $\calP_3=y^3+\alpha y^2+\beta y $ and $\calP_2=y^2+\gamma y$ such that we can still construct supersymmetric tensors involving the regressors. In order to obtain the unique set of coefficients $(\alpha,\beta, \gamma)$, we need to solve a linear system of equations, which we describe in \prettyref{app:nonlinear}.
\end{proof}

Once we obtain $\calT_2$ and $\calT_3$, the recovery gurantees for the regressors $\ba_i^\ast$ follow from the standard tensor decomposition guarantees, for example, Theorem 4.3 and Theorem 5 of \cite{AFHK+12}. We assume that the learnt regressors $\ba_i$ are such that $\max_{i \in [k]}\norm{\ba_i-\ba_i^\ast}_2=\sigma^2 \varepsilon$ for some $\varepsilon>0$. Now we present our theoretical results for global convergence of EM. First we briefly recall the algorithm. Let $\Omega$ denote the domain of our gating parameters, defined as
\begin{align*}
\Omega = \sth{\bw=(\bw_1,\ldots,\bw_{k-1}): \norm{\bw_i}_2 \leq R, \forall i \in [k-1]}.
\end{align*}
Then the population EM for the mixture of experts consists of the following two steps:
\begin{itemize}
\item \textbf{E-step:} Using the current estimate $\bw_t$
 to compute the function $Q(\cdot|\bw_t)$,
\item \textbf{M-step:} $\bw_{t+1}= \mathrm{argmax}_{\bw \in \Omega} Q(\bw|\bw_t)$,
\end{itemize}
where the function $Q(.|\bw_t)$ is the expected log-likelihood of the complete data distribution with respect to current posterior distribution. Mathematically,
\begin{align}
Q(\bw|\bw_t) &\define \Expect_{(\bx,y)} \Expect_{P_{z|\bx,y,\bw_t}} [ \log P_{\bw} (\bx,z,y)  ] \nonumber \\
&=\Expect_{(\bx,y)} \Expect_{P_{z|\bx,y,\bw_t}} [ \log P(\bx)P_{\bw}(z|\bx)P(y|\bx,z)  ] \nonumber \\
&=\Expect_{(\bx,y)} \Expect_{P_{z|\bx,y,\bw_t}}[\log P_{\bw}(z|x)]+ \mathrm{const.} \nonumber \\
&= \Expect[ \sum_{i \in [k-1]} p_{\bw_t}^{(i)} (\bw_i^\top \bx) -  (1+\sum_{i \in [k-1]} e^{\bw_i^\top \bx})]\nonumber\\
&\hspace{5em}+ \mathrm{const.} \label{eq:Qdef}
\end{align}
where $\mathrm{const}$ refers to terms not depending on $\bw$, $P_{\bw}(z=i|x)=\exp(\bw_i^\top \bx)/\sum_j \exp(\bw_j^\top \bx)$ and $p_{\bw_t}^{(i)} \triangleq \prob{z=i|\bx,y,\bw_t}$ corresponds to the posterior probability for the $i^{\text{th}}$ expert, given by
\begin{align}
p_{\bw_t}^{(i)} = \frac{ p_{i,t}(\bx) \calN(y|g(\ba_i^\top \bx),\sigma^2)}{\sum_{j \in [k] } p_{j,t}(\bx) \calN(y|g(\ba_j^\top \bx),\sigma^2)}, \label{eq:computep}\\
 p_{i,t}(\bx)= \frac{e^{(\bw_t)_i^\top \bx}}{1+\sum_{ j \in [k-1]}e^{(\bw_t)_j^\top \bx}}. \nonumber
\end{align}

%
In \prettyref{eq:Qdef}, the expectation is with respect to the true distribution of $(\bx,y)$, given by \prettyref{eq:kmoe}. Thus the EM can be viewed as a deterministic procedure which maps $\bw_t \mapsto M(\bw_t)$ where
\begin{align*}
M(\bw)=\mathrm{argmax}_{ \bw' \in \Omega} Q(\bw'|\bw).
\end{align*}
When the estimated regressors $\ba_i$ equal the true parameters $\ba_i^\ast$, it follows from the self-consistency property of the EM that the true parameter $\bw^\ast$ is a fixed-point for the EM operator $M$, \ie $M(\bw^\ast)=\bw^\ast$ \citep{mclachlan2007algorithm}. However, this does not guarantee that EM converges to $\bw^\ast$. In the following theorem, we show that even when the regressors are known approximately, EM algorithm converges to the true gating parameters at a geometric rate upto an additive error, under \emph{global} initializations. For the error metric, we define $\norm{\bw-\bw'} \define \max_{i \in [k-1]}\norm{\bw_i-\bw'_i}_2$ for any $\bw,\bw' \in \Omega$. We assume that $R=1$ for simplicity. (Our results extend straightforwardly to general $R$).
\begin{theorem}
\label{thm:twoexpertsEM}
Let $\varepsilon>0$ be such that $\max_{i}\|\ba_i-\ba_i^\ast\|_2=\sigma^2 \varepsilon$. There exists a constant $\sigma_0 >0$ such that whenever $0< \sigma<\sigma_0$, for any random initialization $\bw_0 \in \Omega$, the population-level EM updates on the gating parameter $\{\bw\}_{t \geq 0}$ converge almost geometrically to the true parameter $\bw^\ast$ upto an additive error, \ie
\begin{align*}
\norm{\bw_t-\bw^\ast} \leq \pth{\kappa_\sigma}^t \norm{\bw_0 -\bw^\ast} + \kappa\varepsilon \sum_{i=0}^{t-1}\kappa_\sigma^i ,
\end{align*}
where $\kappa_\sigma, \kappa$ are dimension-independent constant depending on $g$ and $\sigma$ such that $\kappa_\sigma \xrightarrow{\sigma\rightarrow 0} 0$ and $\kappa \leq  (k-1)\frac{\sqrt{6(2+\sigma^2)}}{2}$ for $g=$linear, sigmoid and ReLU.
\end{theorem}
\begin{proof}(Sketch) One can show that the $Q(\cdot|\bw_t)$ defined in \prettyref{eq:Qdef} is a strongly concave function. Moreover, if we let $\varepsilon=0$ and $\bw_t=\bw^\ast$, we have from the self-consistency of EM that $\mathrm{argmax} Q(\cdot|\bw^\ast)=\bw^\ast$. Thus if we can show that the functions are $Q(\cdot|\bw_t)$ and $Q(\cdot|\bw^\ast)$ ``sufficiently close" whenever $\bw_t$ and $\bw^\ast$ are close, we can use the EM convergence analysis tools from \cite{bala17} to show that their corresponding maximizers also stay close upto a scaling factor determined by $\kappa_\sigma$ above. Then it follows that the EM updates converge geometrically. 
\end{proof}
\textbf{Remark.} In the M-step of the EM algorithm, the next iterate is chosen so that the function $Q(\cdot|\bw_t)$ is maximized. Instead we can perform an ascent step in the direction of the gradient of $Q(\cdot|\bw_t)$ to produce the next iterate, i.e. $\bw_{t+1}=\Pi_{\Omega}(\bw_t+\alpha\nabla Q(\bw_t|\bw_t) )$, where $\Pi_{\Omega}(\cdot)$ is the projection operator. This variant of EM algorithm is known as \emph{Gradient EM}. In \prettyref{app:twoexpertsgradEM}, we show that Gradient EM also enjoys similar convergence guarantees.


\textbf{MoM to learn gating parameters.} In \prettyref{thm:twoexpertsEM}, we proved that EM algorithm provably recovers the true gating parameters for any $k\geq 2$ mixtures. In this section, we show that for the special case of $k=2$, we can learn $\bw^\ast$ (upto the unit direction) using an alternative procedure involving MoM. First we define
\begin{align}
\mathrm{Ratio}(\bx,y) \define  \frac{  y  - \inner{\ba_2}{\bx}}{\inner{\ba_1-\ba_2}{\bx}}
\end{align} 

The following theorem establishes that the the CDF of the random variable $\mathrm{Ratio}(\bx,y)$, when regressed on input $\bx$, is proportional to $\bw^\ast$.

\begin{theorem}
Suppose that $(\ba_1,\ba_2)=(\ba_1^\ast,\ba_2^\ast)$. Then we have that 
\begin{align*}
\Expect[\mathds{1}\{\mathrm{Ratio}(\bx,y) \leq 0.5\}\cdot \bx]= \alpha \bw^\ast,
\end{align*}
where $\alpha \in \reals$ is a scalar given by $\alpha=\Expect[f'(\inner{\bw^\ast}{\bx}) \cdot \pth{1-2\Phi \pth{\frac{|\inner{\ba_1-\ba_2}{\bx}|}{2\sigma}}} ]  $.
\label{thm:MoM_gating}
\end{theorem}
\begin{proof}(Sketch)
We first show that $\mathrm{Ratio}(\bx,y)$ is a mixture of Cauchy distributions. Then we show that $\Expect[\mathds{1}\{\mathrm{Ratio}(\bx,y) \leq z \}|\bx]= \prob{\mathrm{Ratio} \leq z|\bx} = f(\bw^\top\bx) \Phi \pth{(z-1) \frac{|\Delta_x|}{\sigma}}+(1-f(\bw^\top\bx)) \Phi \pth{z \frac{|\Delta_x|}{\sigma}}$ where $\Delta_x=(\ba_1-\ba_2)^\top \bx$. Then our result follows from taking the first moment of the indicator random variable with $\bx$ and Stein's lemma.
\end{proof}

\section{Experiments}
   \label{sec:experiments}
In this section, we empirically validate our algorithm in various settings and compare its performance to that of EM on both synthetic and real world datasets \footnote{Codes are available at this repository \href{https://github.com/Ashokvardhan/Breaking-the-gridlock-in-MoE-Consistent-and-Efficient-Algorithms}{MoE codes}.}. In both the scenarios, we found that our algorithm consistently outperforms the existing approaches. For the tensor decomposition in our \prettyref{algo:regressrecov}, we use the Orth-ALS package by \cite{orthals}. In all the synthetic experiments, we first draw the regressors $\{\ba_{i}^\ast\}_{i=1}^k$ i.i.d uniformly from the unit sphere $\mathbb{S}^{d-1}$. The input distribution $P_{\bx}$ and the generation of $\bw_i^\ast$'s are detailed for each experiment. Then the labels $y_i$ are generated according to the true $k$-MoE model in \prettyref{eq:kmoe} for linear activation. Additional experiments in this setting with non-linear activations are detailed in \prettyref{app:app_synthetic_data}. Experiments with real world data are provided in \prettyref{sec:real_data}.

\subsection{Non-gaussian inputs}
\label{sec:non-gaussian}
In this section we let the input distribution to be mixtures of Gaussians (GMM). We let $k=2, d=10$ and $\sigma=0.1$. The gating parameter $\bw^\ast \in \reals^{10}$ is uniformly chosen from the unit sphere $\mathbb{S}^9$. To generate the input features, we first randomly draw $\mu_1,\mu_2 \in \mathbb{S}^9$, and generate $n$ \iid samples $\bx_i \sim p \calN(\mu_1,I_d)+(1-p)\calN(\mu_2,I_d)$, where $p\in \{0.1,0.3,0.5,0.7,0.9 \}$. Here $n=2000$. Since $\bx$ is a $2$-GMM, its score functions $\calS_3(\bx),\calS_2(\bx)$ are computed using the densities of Gaussian mixtures \citep{scorebusiness}. To gauge the performance of our algorithm, we measure the correlation of our learned parameters $\ba_1,\ba_2$ and $\bw$ with the ground truth, \ie
\begin{align}
\mathsf{Regressor \ Fit}(\ba_1,\ba_2) = \max_{\pi} \min_{i \in \{1,2\}} |\inner{\ba_{\pi(i)}}{\ba_i^\ast}|,
\label{eq:regressorfit}
\end{align}
where $\pi:\{1,2\}\to \{1,2\}$ is a permutation. Similarly, for the gating parameter, we define
\begin{align}
\mathsf{Gating \ Fit}(\bw) = |\inner{\bw}{\bw^\ast}|.
\label{eq:gatingfit}
\end{align}
Here we assume that all the parameters are unit-normalized. The closer the values of fit are to $1$, the closer the learnt parameters are to the ground truth. As shown in \prettyref{tab:hello}, our algorithms are able to learn the ground truth very accurately in a variety of settings, as indicated by the measured fit. This highlights the fact that our algorithms are robust to the input distributions.

\subsection{Non-orthogonal parameters}
\label{sec:non-orthogonal}
In this section we verify that our algorithms still work well in practice even under the relaxation of Assumption $3$. For the experiments, we consider the similar setting as before with $k=2,d=10,\sigma=0.1$ and the gating parameter $\bw^\ast$ is drawn uniformly from $\mathbb{S}^9$ without the orthogonality restriction. We let $\bx_i \stackrel{\iid}\sim \calN(0,I_d)$. We choose $n=2000$. We use $\mathsf{RegressorFit}$ and  $\mathsf{GatingFit}$ defined in \prettyref{eq:regressorfit} and \prettyref{eq:gatingfit} respectively, as our performance metrics. From \prettyref{tab:non_orthogonal}, we can see that the performance of our algorithms is almost the same across both the settings. In both the scenarios, our fit is consistently greater than $0.9$. 

In \prettyref{fig:nonorthogonalEMw}, we plotted $\mathsf{GatingFit}(\bw_t)$ vs. the number of iterations $t$, as $\bw_t$ is updated according to \prettyref{algo:gatingrecov}, over $10$ independent trials. We observe that the learned parameters converge to the true parameters in less than $5$ iterations.

\begin{figure*}
    \centering
    \begin{subfigure}[b]{0.32\textwidth}
        \includegraphics[width=\textwidth]{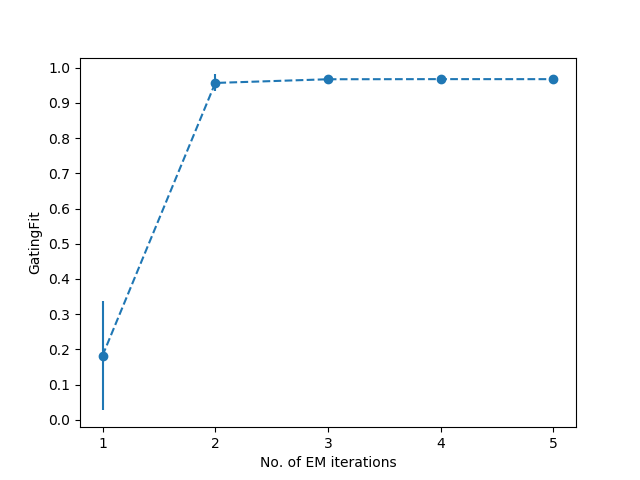}
        \caption{Non-orthogonality}
        \label{fig:nonorthogonalEMw}
    \end{subfigure}
    \begin{subfigure}[b]{0.32\textwidth}
        \includegraphics[width=\textwidth]{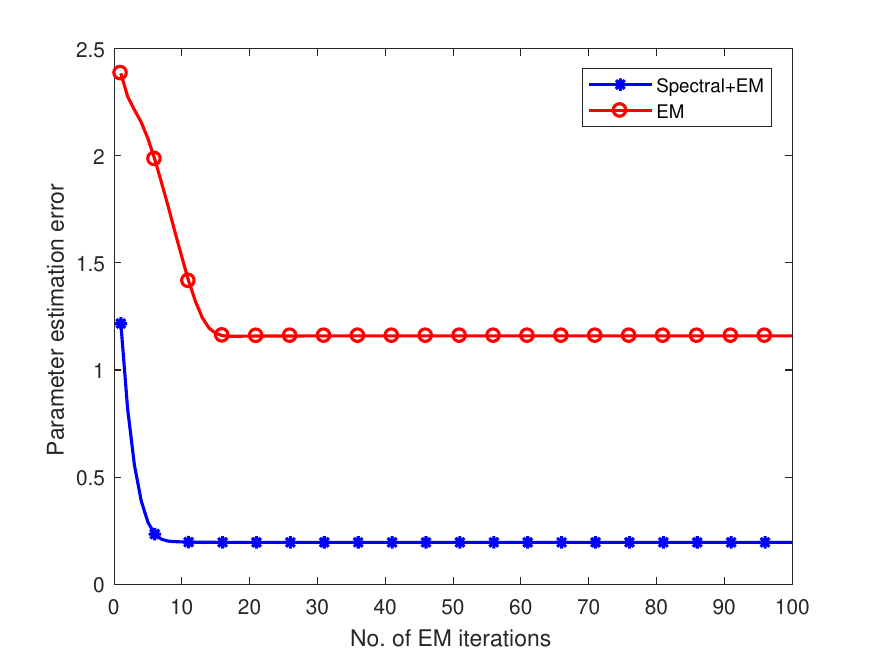}
        \caption{$k=3$}
        \label{fig:k_3}
    \end{subfigure}
    ~ 
    \begin{subfigure}[b]{0.32\textwidth}
        \includegraphics[width=\textwidth]{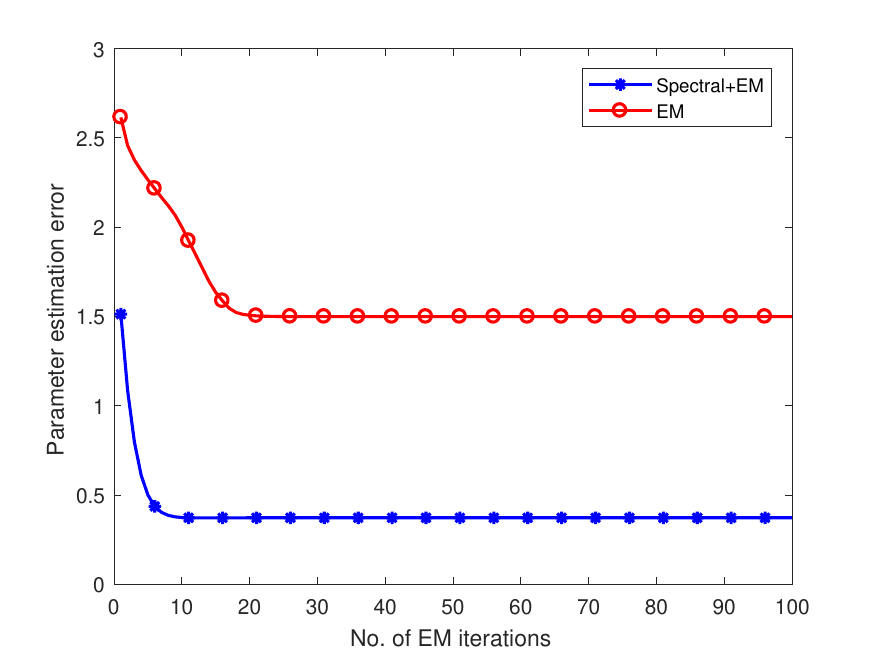}
        \caption{$k=4$}
        \label{fig:k_4}
    \end{subfigure}
    \caption{(a): $\mathsf{GatingFit}$ for our algorithm under non-orthogonality setting. (b),(c): Estimation error $\calE(\bA,\bW)$ of our algorithm vs. joint-EM algorithm. Our algorithm is significantly better than the joint-EM under random initializations.}\label{fig:animals}
\end{figure*}

\begin{table*}[ht]
\caption{Fit of our learned parameters for non-Gaussian inputs}
\centering
\begin{tabular}{cccccc}
 \hline 
 & $p=0.1$ & $p=0.3$  & $p=0.5$ & $p=0.7$ & $p=0.9$ \\
 \hline 
$\mathsf{Regressor \ Fit}$ & $0.93 \pm 0.06$ & $0.94 \pm 0.02$  & $0.92 \pm 0.04$ & $0.92 \pm 0.02$ & $0.91 \pm 0.06$ \\
 \hline 
$\mathsf{Gating \ Fit}$& $0.9 \pm 0.1$ & $0.97 \pm 0.01$  & $0.93 \pm 0.04$ & $0.96 \pm 0.03$ & $0.97 \pm 0.01$ \\
 \hline 
 \end{tabular} 

\label{tab:hello}

\bigskip

\caption{Performance of our algorithm under orthogonal and non-orthogonal settings}
\centering
\begin{tabular}{ccc}
 \hline 
 & $\mathsf{Regressor \ Fit}$ & $\mathsf{Gating \ Fit}$  \\
 \hline 
Non-orthogonal & $0.9 \pm 0.08 $ & $0.96\pm 0.02$  \\
 \hline 
Orthogonal & $0.93 \pm 0.03$ & $0.96 \pm 0.03$  \\
 \hline 
 \end{tabular} 

\label{tab:non_orthogonal}

\end{table*}

\subsection{Comparison to joint-EM}
\label{sec:compjoint-EM}
Here we compare the performance of our algorithm with that of the joint-EM. We let the number of mixture components be $k=3$ and $k=4$. We let $\bx \sim \calN(0,I_d)$ and the gating parameters are drawn uniformly from $\mathbb{S}^9$. If $\bA=[\ba_1 \ldots \ba_k]$ and $\bW=[\bw_1 \ldots \bw_{k-1} \ 0]$ denote the estimated expert and gating parameters respectively, our evaluation metric is $\calE$, the Frobenious norm of the parameter error accounting for the best possible permutation $\pi:[k] \to [k]$, \ie $\calE(\bA,\bW) =\inf_{\pi} \| \bA-\bA_{\pi}^\ast \|_F+\| \bW -\bW_{\pi}^\ast \|_F$, where $\bA_\pi^\ast=[\ba^\ast_{\pi(1)} \ldots \ba^\ast_{\pi(k)}]$ denotes the permuted regression parameter matrix and similarly for $\bW_\pi^\ast$. In \prettyref{fig:k_3} and \prettyref{fig:k_4}, we compare the performance of our algorithm with the joint-EM algorithm for $n=8000, d=10, \sigma=0.5$. The plotted estimation error $\calE(\bA,\bW)$ is averaged for $10$ trials. It is clear that our algorithm is able to recover the true parameters thus resulting in much smaller parameter error than the joint-EM which often gets stuck in local optima. In addition, our algorithm is able to learn these parameters in very few iterations, often less than $10$ iterations. We also find that our algorithm consistently outperforms the joint-EM for different choices of non-linearities, number of samples, number of mixtures, etc. (details provided in \prettyref{app:additionalexp}). Note that the above error metric $\calE(\bA,\bW)$ is close to zero if and only if $\mathsf{Regressor \ Fit}$ and $\mathsf{Gating \ Fit}$ is close to one.

\subsection{Real data}
\label{sec:real_data}
To highlight the generalizability of our algorithm, in \prettyref{app:app_real_data} of the supplement, we compare the performance of our algorithm to that of the standard approaches on a variety of real world datasets. Results from these experiments highlight the fact that in the real world scenario, where the underlying data is not generated according to a MoE model, our approach still learns a superior set of parameters as opposed to the existing algorithms. This fact is reflected in the lowest prediction errors obtained by our algorithm.

\section{Discussion}
\label{sec:discussion}
In this paper we provided the first provable and globally consistent algorithm that can learn the true parameters of a MoE model. We believe that ideas from \cite{SedghiA14a} can be naturally extended for the finite sample complexity analysis of the tensor decomposition to learn the regressors and similarly, techniques from \cite{bala17} can be extended to the finite sample EM convergence analysis for the gating parameters. While we have focused here on parameter recovery, however, there are no statistical bounds on output prediction error when the data is not generated from the model. MoE models are known to be capable of fitting general functions, and getting statistical guarantees on learning in such regimes is an interesting direction for future work.

\section*{Acknowledgements}
\label{sec:acknow}
This work is partly supported by NSF grants 1927712 and 1815535, NSF
awards CNS-1718270, 1651236, 1703403, and the Army Research Office
under grant W911NF1810332.

\bibliography{references}
\bibliographystyle{icml2019}

%
%
%

\appendix
\onecolumn

\textbf{Organization.} The appendix is organized as follows:
\begin{itemize}
\item \prettyref{app:reviewtensor} and \prettyref{app:reviewem} contain the requisite material for method of moments and the convergence analysis of EM respectively. 

\item \prettyref{app:nonlinear} details the class of non-linearities for which our results hold.

\item  \prettyref{app:twoproofs} contains all the proofs of \prettyref{sec:guaranteestwomoe}. Two technical lemmas needed to prove \prettyref{thm:twoexpertsEM} are relegated to \prettyref{app:twoexpertsperturbation} and \prettyref{app:proofcontraction}.

\item \prettyref{app:twoexpertsgradEM} provides convergence guarantees for Gradient EM.


\item \prettyref{app:additionalexp} contains additional experiments for the comparison of joint-EM and our algorithm for the synthetic data.

\end{itemize}

\section{Toolbox for method of moments}
\label{app:reviewtensor}
In this section, we introduce the key techniques that are useful in parameter estimation of mixture models via the method of moments. 

Stein's identity (Stein's lemma) is a well-known result in probability and statistics and is widely used in estimation and inference taks. A refined version of the Stein's lemma \cite{Ste72} for higher-order moments is the key to parameter estimation in mixture of generalized linear models. We utilize this machinery in proving \prettyref{thm:twoexpertspopulation}. We first recall the Stein's lemma.

\begin{lemma}[Stein's lemma \cite{Ste72} ]
\label{lmm:stein} Let $\bx \sim \calN(0,I_d)$ and $g:\reals^d \to \reals$ be a function such that both $\Expect[\nabla_{\bx} g(\bx)]$ and $\Expect[g(\bx)\cdot \bx]$ exist and are finite. Then
\begin{align*}
\Expect[g(\bx) \cdot \bx]= \Expect[ \nabla_{\bx} g(\bx)].
\end{align*}
\end{lemma}

The following lemma, which can be viewed as an extension of Stein's lemma for higher-order moments, is the central technique behind parameter estimation in M-GLMs.
\begin{lemma}[\cite{SedghiA14a}]
\label{lmm:highstein}
Let $\bx \sim \calN(0,I_d)$ and $\calS_3(\bx)$ be as defined in \prettyref{eq:thirdscore} and let $\calS_2(\bx) \triangleq \bx \otimes \bx -I_d$. Then for any $g:\reals^d \to \reals$ satisfying some regularity conditions, we have
\begin{align*}
\Expect[g(\bx) \cdot \calS_2(\bx)]= \Expect[\nabla_{\bx}^{(2)} g(\bx)], \quad 
\Expect[g(\bx) \cdot \calS_3(\bx)]= \Expect[\nabla_{\bx}^{(3)} g(\bx)].
\end{align*}
\end{lemma}

%

\section{Toolbox for EM convergence analysis}
\label{app:reviewem}

Recall that the domain of our gating parameters is $\Omega= \{\bw:\norm{\bw} \leq 1\}$. Then the population EM for the mixture of experts consists of the following two steps:
\begin{itemize}
\item \textbf{E-step:} Using the current estimate $\bw_t$
 to compute the function $Q(\cdot|\bw_t)$.
\item \textbf{M-step:} $\bw_{t+1}= \mathrm{argmax}_{\norm{\bw}\leq 1}Q(\bw|\bw_t)$.
\end{itemize}
Thus the EM can be viewed as a deterministic procedure which maps $\bw_t \mapsto M(\bw_t)$ where
\begin{align*}
M(\bw)=\mathrm{argmax}_{ \bw' \in \Omega} Q(\bw'|\bw).
\end{align*}

Our convergence analysis relies on tools from \cite{bala17} where they provided local convergence results on both the EM and gradient EM algorithms.
In particular, they showed that if we initialize EM in a sufficiently small neighborhood around the true parameters, the EM iterates converge geometrically to the true parameters under some strong-concavity and gradient stability conditions.
We now formally state the assumptions in \cite{bala17} under which the convergence guarantees hold. We will show in the next section that these conditions hold \emph{globally} in our setting. 
\begin{assumption}[Convexity of the domain]
\label{assump:A1}
\normalfont $\Omega$ is convex.
\end{assumption}
\begin{assumption}[Strong-concavity]
\label{assump:A2}
\normalfont $Q(\cdot|\bw^\ast)$ is a $\lambda$-strongly concave function over a $r$-neighborhood of $\bw^\ast$, \ie $\calB(\bw^\ast,r)\triangleq \{\bw \in \Omega: \norm{\bw-\bw^\ast} \leq r\}$.
\end{assumption}
\begin{remark}\normalfont
An important point to note is that the true parameter $\bw^\ast$ is a fixed point for the EM algorithm, \ie $M(\bw^\ast)=\bw^\ast$. This is also known as \emph{self-consistency} of the EM algorithm. Hence it is reasonable to expect that in a sufficiently small neighborhood around $\bw^\ast$ there exists a unique maximizer for $Q(\cdot|\bw^\ast)$.
\end{remark}
\begin{assumption}[First-order stability condition]
\label{assump:A3}
\normalfont Assume that 
\begin{align*}
\norm{\nabla Q(M(\bw)|\bw^\ast) - \nabla Q(M(\bw)|\bw)} \leq \gamma \norm{\bw -\bw^\ast}, \quad \forall \bw \in \calB(\bw^\ast,r).
\end{align*}
\end{assumption}
\begin{remark}\normalfont
Intuitively, the gradient stability condition enforces the gradient maps $\nabla Q(\cdot|\bw)$ and $\nabla Q(\cdot|\bw^\ast)$ to be close whenever $\bw$ lies in a neighborhood of $\bw^\ast$. This will ensure that the mapped output $M(\bw)$ stays closer to $\bw^\ast$. 
\end{remark}
\begin{theorem}[Theorem 1, \cite{bala17}]
If the above assumptions are met for some radius $r>0$ and $0 \leq \gamma < \lambda$, then the map $\bw \mapsto M(\bw)$ is contractive over $\calB(\bw^\ast,r)$, \ie
\begin{align*}
\norm{M(\bw)-\bw^\ast} \leq \pth{\frac{\gamma}{\lambda}} \norm{\bw-\bw^\ast}, \quad \forall \bw \in \calB(\bw^\ast,r),
\end{align*}
and consequently, the EM iterates $\{\bw_t\}_{t \geq 0}$ converge geometrically to $\bw^\ast$, \ie
\begin{align*}
\norm{\bw_t-\bw^\ast} \leq \pth{\frac{\gamma}{\lambda}}^t \norm{\bw_0-\bw^\ast},
\end{align*}
whenever the initialization $\bw_0 \in \calB(\bw^\ast,r)$.
\end{theorem}

\section{Class of non-linearities}
\label{app:nonlinear}
In this section, we characterize the class of non-linearities for which our theoretical results for the recovery of regressors hold. Let $Z \sim \calN(0,1)$ and $Y|Z \sim \calN(g(Z),\sigma^2)$, where $g: \reals \to \reals$. For $(\alpha, \beta, \gamma ) \in \reals^3$, define
\begin{align*}
\calP_3(y) &\triangleq Y^3+\alpha Y^2+\beta Y, \quad  \calS_3(Z)= \Expect[ \calP_3(y)|Z]=g(Z)^3+\alpha g(Z)^2+ g(Z)(\beta+3 \sigma^2)+\alpha \sigma^2,
\end{align*}
and
\begin{align*}
\calS_2(Y) &\triangleq Y^2+\gamma Y, \quad  \calS_2(Z)= \Expect[ \calS_2(Y)|Z]=g(Z)^2+ \gamma g(Z)+\sigma^2.
\end{align*}
\begin{condition}\normalfont
\label{cond:cond1}$\Expect[\calS_3'(Z)]=\Expect[\calS_3''(Z)]=0$ and $\Expect[\calS_3'''(Z)] \neq 0$.
\end{condition}
\begin{condition}\normalfont
\label{cond:cond2}$\Expect[\calS_2'(Z)]=0$ and $\Expect[\calS_2''(Z)] \neq 0$.
\end{condition}

We are now ready to define the $(\alpha,\beta,\gamma)$-valid class of non-linearities. 

\begin{definition}
\label{def:valid}
We say that the non-linearity $g$ is $(\alpha,\beta,\gamma)$-valid if there exists $(\alpha,\beta,\gamma) \in \reals^3$ such that both  \prettyref{cond:cond1} and \prettyref{cond:cond2} are satisfied.
\end{definition}
We have that
\begin{align*}
\calS_3'(Z) &= 3g(Z)^2 g'(Z) + 2 \alpha g(Z) g'(Z)+g'(Z)(\beta+3 \sigma^2)\\
&=2 \alpha g(Z) g'(Z) + \beta g'(Z) + 3g(Z)^2 g'(Z)+3g'(Z)\sigma^2,\\
\calS_3''(Z) &=2 \alpha \pth{g'(Z)^2+g(Z)g''(Z)} + \beta g''(Z)+3g''(Z)(g(Z)^2+\sigma^2)+6 g(Z) g'(Z)^2.
\end{align*}
Thus $\Expect[\calS_3'(Z)]=\Expect[\calS_3''(Z)]=0$ implies that
\begin{align*}
\begin{bmatrix}
2  \Expect(g(Z) g'(Z)) & \Expect( g'(Z))  \\
2 \Expect \pth{g'(Z)^2+g(Z)g''(Z)} & \Expect( g''(Z))
\end{bmatrix} \begin{bmatrix}
\alpha \\ \beta
\end{bmatrix} = \begin{bmatrix}
  -3\Expect(g(Z)^2 g'(Z)+g'(Z)\sigma^2 ) \\
  -3\Expect (g''(Z)(g(Z)^2+\sigma^2)+2 g(Z) g'(Z)^2)
\end{bmatrix}
\end{align*}
To ensure \prettyref{cond:cond1}, we need the pair $(\alpha,\beta)$ obtained by solving the above linear equation to satisfy $\Expect[\calS_3'''(Z)] \neq 0$. Similarly, $\Expect[\calS_2'(Z)]=0$ implies that
\begin{align*}
\gamma = \frac{-2 \Expect[g(Z)g'(Z)]}{\Expect[g'(Z)]}.
\end{align*}
Thus \prettyref{cond:cond2} stipulates that $\Expect[\calS_2''(Z)] \neq 0$ with this choice of $\gamma$. It turns out that these conditions hold for a wide class of non-linearities and in particular, when $g$ is either the identity function, or the sigmoid function, or the ReLU. For these three choices of popular non-linearities, the values of the tuple $(\alpha,\beta,\gamma)$ are  provided below (which are obtained by solving the linear equations mentioned above).
\begin{example}\normalfont
If $g$ is the identity mapping, then $\calP_3(y)=y^3-3y(1+\sigma^2)$ and $\calS_2(y)=y^2$.
\end{example}
\begin{example}\normalfont
If $g$ is the sigmoid function, \ie $g(z)=\frac{1}{1+e^{-z}}$, then $\alpha$ and $\beta$ can be obtained by solving the following linear equation:
\begin{align*}
\begin{bmatrix}
0.2066 & 0.2066 \\
0.0624 & -0.0001
\end{bmatrix} 
\begin{bmatrix}
\alpha \\ \beta
\end{bmatrix}= \begin{bmatrix}
-0.1755-0.6199 \sigma^2 \\ -0.0936
\end{bmatrix}
\end{align*}
The second-order transformation is given by $\calS_2(y)=y^2-y$ (since $\gamma=-1$ when $g$ is sigmoid).
\end{example}
\begin{example}\normalfont
If $g$ is the ReLU function, \ie $g(z)=\mathrm{max}\{0,z\}$, then $\alpha=-3\sqrt{\frac{2}{\pi}}, \beta=3\pth{\frac{4}{\pi}-\sigma^2-1}$ and $\gamma=-2 \sqrt{\frac{2}{\pi}}$.
\end{example}

\section{Proofs of \prettyref{sec:guaranteestwomoe}}
\label{app:twoproofs}

In this section, for the simplicity of the notation we denote the true parameters as $\bw_i$'s and $\ba_i$'s dropping the $\ast$ sign.


\subsection{Proof of \prettyref{thm:twoexpertspopulation} for $k=2$}
\label{app:regressorlinear}
\begin{proof}
Suppose that $g$ is the linear activation function. For $k=2$, \prettyref{eq:kmoe} implies that
\begin{align}
P_{y|\bx}= f(\bw^\top \bx) \cdot \calN(y|\ba_1^\top \bx, \sigma^2)+(1-f(\bw^\top \bx))\cdot\calN(y|\ba_2^\top \bx,\sigma^2), \quad \bx \sim \calN(0,I_d), \label{eq:twolinearmodel}
\end{align}   
where $f(\cdot)$ is the sigmoid function. Using the fact $\Expect[Z^3]=\mu^3+3\mu\sigma^2$ for any Gaussian random variable $Z \sim \calN(\mu,\sigma^2)$, we get
\begin{align*}
\Expect[y^3|\bx]=f(\bw^\top \bx) ((\ba_1^\top \bx)^3+3(\ba_1^\top \bx)\sigma^2)+(1-f(\bw^\top \bx))((\ba_1^\top \bx)^3+3(\ba_1^\top \bx)\sigma^2).
\end{align*}    
Moreover,
\begin{align*}
\Expect[y|\bx]=f(\bw^\top \bx)(\ba_1^\top \bx)+(1-f(\bw^\top \bx))(\ba_2^\top \bx).
\end{align*}    
Thus,
\begin{align*}
\Expect[y^3-3y(1+\sigma^2)|\bx]=f(\bw^\top \bx) ((\ba_1^\top \bx)^3-3(\ba_1^\top \bx))+(1-f(\bw^\top \bx))((\ba_1^\top \bx)^3-3(\ba_1^\top \bx)).
\end{align*}    
If we define $\calP_3(y)\triangleq y^3-3y(1+\sigma^2)$, in view of \prettyref{lmm:highstein} we get that
\begin{align}
\calT_3 = \Expect[\calP_3(y) \cdot \calS_3(\bx)] &= \Expect[(y^3-3y(1+\sigma^2)) \cdot \calS_3(\bx)] \nonumber \\
&= \Expect \qth{\pth{f(\bw^\top \bx) ((\ba_1^\top \bx)^3-3(\ba_1^\top \bx))}\cdot \calS_3(\bx)}+\Expect \qth{\pth{1-f(\bw^\top \bx) ((\ba_2^\top \bx)^3-3(\ba_2^\top \bx))}\cdot \calS_3(\bx)}\nonumber\\
&=\Expect\qth{\nabla_{\bx}^{(3)} \pth{{f(\bw^\top \bx) ((\ba_1^\top \bx)^3-3(\ba_1^\top \bx))}}}+\Expect\qth{\nabla_{\bx}^{(3)} \pth{{1-f(\bw^\top \bx) ((\ba_2^\top \bx)^3-3(\ba_2^\top \bx))}}}. \label{eq:twotermsfirst}
\end{align}
Using the chain rule for multi-derivatives, the first term simplifies to
\begin{align}
\Expect\qth{\nabla_{\bx}^{(3)} \pth{{f(\bw^\top \bx) ((\ba_1^\top \bx)^3-3(\ba_1^\top \bx))}}}=\Expect [f'''  ((\ba_1^\top \bx)^3-3(\ba_1^\top \bx))] \cdot \bw \otimes \bw \otimes \bw  +\Expect[f''(3(\ba_1^\top \bx)^2-3) ]\cdot \nonumber \\
(\bw \otimes \bw \otimes \ba_1+\bw \otimes \ba_1 \otimes \bw +\ba_1 \otimes \bw \otimes \bw) + \nonumber \\
\Expect[f' (6(\ba_1^\top \bx))] \cdot \pth{\ba_1 \otimes \ba_1 \otimes \bw+ \ba_1 \otimes \bw \otimes \ba_1 + \bw \otimes \ba_1 \otimes \ba_1}+ 6 \Expect[f ] \cdot \ba_1 \otimes \ba_1 \otimes \ba_1. \label{eq:oddevenfirst}
\end{align}
Since $f(z)=\frac{1}{1+e^{-z}}$, $f'(\cdot),f'''(\cdot)$ are even functions whereas $f''(\cdot)$ is an odd function. Furthermore, both $\bx \mapsto (\ba_1^\top \bx)^3-3(\ba_1^\top \bx)$ and $\bx \mapsto \ba_1^\top \bx$ are odd functions whereas $ \bx \mapsto 3(\ba_1^\top \bx)^2-3 $ is an even function. Since $\bx \sim \calN(0,I_d)$, $-\bx \stackrel{(d)}=  \bx$. Thus all the expectation terms in \prettyref{eq:oddevenfirst} equal zero except for the last term since $\Expect[f(\bw^\top \bx)]=\frac{1}{2} >0$. We have,
\begin{align*}
\Expect\qth{\nabla_{\bx}^{(3)} \pth{{f(\bw^\top \bx) ((\ba_1^\top \bx)^3-3(\ba_1^\top \bx))}}}= 3 \cdot \ba_1 \otimes \ba_1 \otimes \ba_1.
\end{align*}
Similarly,
\begin{align*}
\Expect\qth{\nabla_{\bx}^{(3)} \pth{{1-f(\bw^\top \bx) ((\ba_2^\top \bx)^3-3(\ba_2^\top \bx))}}}=3 \cdot \ba_2 \otimes \ba_2 \otimes \ba_2.
\end{align*}
Together, we have that
\begin{align*}
\calT_3 =3 \cdot \ba_1 \otimes \ba_1 \otimes \ba_1+3 \cdot \ba_2 \otimes \ba_2 \otimes \ba_2.
\end{align*}

Now consider an arbitrary link function $g$ belonging to the class of non-linearities described in \prettyref{app:nonlinear}. Then 
\begin{align*}
P_{y|\bx}= f(\bw^\top \bx) \cdot \calN(y|g(\ba_1^\top \bx), \sigma^2)+(1-f(\bw^\top \bx))\cdot\calN(y|g(\ba_2^\top \bx),\sigma^2), \quad \bx \sim \calN(0,I_d),
\end{align*}
implies that
\begin{align*}
\Expect[y^3|\bx]=f(\bw^\top \bx) (g(\ba_1^\top \bx)^3+3g(\ba_1^\top \bx)\sigma^2)+(1-f(\bw^\top \bx))(g(\ba_2^\top \bx)^3+3g(\ba_2^\top \bx)\sigma^2),
\end{align*} 
and
\begin{align*}
\Expect[y^2|\bx]&=f(\bw^\top \bx) (g(\ba_1^\top \bx)^2+\sigma^2)+(1-f(\bw^\top \bx))(g(\ba_2^\top \bx)^2+\sigma^2),\\
\Expect[y|\bx]&=f(\bw^\top \bx) g(\ba_1^\top \bx)+(1-f(\bw^\top \bx))g(\ba_2^\top \bx).
\end{align*}
If we define $\calP_3(y) \triangleq y^3+\alpha y^2+\beta y$, we have that
\begin{align*}
\calT_3 = \Expect[\calP_3(y)\cdot \calS_3(\bx)] &= \Expect[ \Expect[y^3+\alpha y^2+\beta y|\bx] \cdot \calS_3(\bx)]\\
&=\Expect \qth{ f(\bw^\top \bx) \pth{g(\ba_1^\top \bx)^3+\alpha  g(\ba_1^\top \bx)^2+g(\ba_1^\top \bx)(\beta+3 \sigma^2)} \cdot \calS_3(\bx)}+\\
&\Expect \qth{(1- f(\bw^\top \bx)) \pth{g(\ba_2^\top \bx)^3+\alpha  g(\ba_2^\top \bx)^2+g(\ba_2^\top \bx)(\beta+3 \sigma^2)} \cdot \calS_3(\bx)}\\
&= \Expect \qth{\nabla_{\bx}^{(3)} \pth{f(\bw^\top \bx) \pth{g(\ba_1^\top \bx)^3+\alpha  g(\ba_1^\top \bx)^2+g(\ba_1^\top \bx)(\beta+3 \sigma^2)}}}+\\
&\Expect \qth{\nabla_{\bx}^{(3)} \pth{f(\bw^\top \bx) \pth{g(\ba_2^\top \bx)^3+\alpha  g(\ba_2^\top \bx)^2+g(\ba_2^\top \bx)(\beta+3 \sigma^2)}}}\\
&\stackrel{(a)}=\Expect[f]\Expect \qth{\nabla_{\bx}^{(3)} \pth{g(\ba_1^\top \bx)^3+\alpha  g(\ba_1^\top \bx)^2+g(\ba_1^\top \bx)(\beta+3 \sigma^2)} }\cdot \ba_1 \otimes \ba_1 \otimes \ba_1 +\\
&\Expect[1-f]\Expect \qth{\nabla_{\bx}^{(3)} \pth{g(\ba_2^\top \bx)^3+\alpha  g(\ba_2^\top \bx)^2+g(\ba_2^\top \bx)(\beta+3 \sigma^2)} }\cdot \ba_2 \otimes \ba_2 \otimes \ba_2 \\
&=c_{g,\sigma}\pth{ \Expect[f]\cdot \ba_1 \otimes \ba_1 \otimes \ba_1 + \Expect[1-f] \cdot \ba_2 \otimes \ba_2 \otimes \ba_2},
\end{align*}
where $(a)$ follows from the choice of $\alpha$ and $\beta$ and the fact that $\bw \perp \{\ba_1,\ba_2\}$, and $c_{g,\sigma}\define \Expect \qth{ \pth{g(Z)^3+\alpha  g(Z)^2+g(Z)(\beta+3 \sigma^2)}''' }$ where $Z \sim \calN(0,1)$ .
The proof for $\calT_2$ is similar.

\end{proof}

\subsection{Proof of \prettyref{thm:twoexpertspopulation} for general $k$}
\label{app:generalregressorproof} 
\begin{proof}
The proof for general $k$ closely follows that of $k=2$, described in \prettyref{app:regressorlinear}. For the general $k$, we first prove the theorem when $g$ is the identity function, \ie
\begin{align*}
P_{y|\bx} &= \sum_{i \in [k]} P_{i|\bx} P_{y|\bx,i} =
 \sum_{i \in [k]} \frac{e^{\bw_i^\top \bx}}{\sum_{i \in [k]}e^{\bw_i^\top \bx}} \cdot \calN(y|\ba_i^\top \bx,\sigma^2),\quad  \bx \sim \calN(0,I_d).
\end{align*}
Denoting $P_{i|\bx}$ by $p_i(\bx)$, we have that
\begin{align*}
\Expect[y^3|\bx] &= \sum_{i \in [k]} p_i(\bx) \pth{(\ba_i^\top \bx)^3+3(\ba_i^\top \bx) \sigma^2},\\
\Expect[y|\bx]&= \sum_{i \in [k]} p_i(\bx) (\ba_i^\top \bx).
\end{align*}
Hence
\begin{align*}
\Expect[y^3-3y(1+\sigma^2)|\bx] &= \sum_{i \in [k]} p_i(\bx) \pth{(\ba_i^\top \bx)^3-3(\ba_i^\top \bx)}
\end{align*}
If we let $\calP_3(y) \triangleq y^3-3y(1+\sigma^2)$, we get
\begin{align*}
\Expect[ \calP_3(y) \cdot \calS_3(\bx)]= \sum_{i \in [k]} \Expect \qth{\nabla_{\bx}^{(3)}\pth{p_i(\bx) \pth{(\ba_i^\top \bx)^3-3(\ba_i^\top \bx)} }}
\end{align*}
Since $\bx \sim \calN(0,I_d)$ and $\ba_i \perp \mathrm{span}\{\bw_1,\ldots,\bw_{k-1}\}$, we have that $\ba_i^\top \bx \perp (\bw_1^\top \bx,\ldots, \bw_{k-1}^\top \bx)$. Moreover,  $\Expect[(\ba_i^\top \bx)^3-3(\ba_i^\top \bx)]=\Expect[(\ba_i^\top \bx)^2-1]=\Expect[\ba_i^\top \bx]=0$ for each $i$. Using the chain-rule for multi-derivatives, the above equation thus simplifies to
\begin{align*}
\Expect[ \calP_3(y) \cdot \calS_3(\bx)]= \sum_{i \in [k]} \Expect[p_i(\bx)]\cdot  \Expect \qth{ \nabla_{\bx}^{(3)}  \pth{(\ba_i^\top \bx)^3-3(\ba_i^\top \bx)} } =  \sum_{i \in [k]} 6 \Expect[p_i(\bx)] \cdot \ba_i \otimes \ba_i \times \ba_i. 
\end{align*}
For a generic $g: \reals \to \reals$ which is $(\alpha,\beta,\gamma)-$valid, let $\calP_3(y)=y^3+\alpha y^2 + \beta y$. Then it is easy to see that the same proof goes through except for a change in the coefficients of rank-$1$ terms, \ie
\begin{align*}
\Expect[\calP_3(y) \cdot \calS_3(\bx)]= \sum_{i \in [k]} \alpha_i \Expect[p_i(\bx)] \cdot \ba_i \otimes  \ba_i \otimes \ba_i,
\end{align*}
where $\alpha_i \triangleq  \Expect \qth{ \pth{g(Z)^3+\alpha  g(Z)^2+g(Z)(\beta+3 \sigma^2)}''' }$ where $Z \sim \calN(0,1)$ and $'''$ denotes the third-derivative with respect to $Z$. Note that \prettyref{cond:cond2} together with the fact that $\Expect[p_i(\bx)]>0$ ensures that $\alpha_i \neq 0$ and thus the coefficients of the rank-$1$ terms are non-zero. The proof for $\calT_2$ is similar.
\end{proof}

\subsection{Proof of \prettyref{thm:twoexpertsEM}}
\label{app:twoexpertsEM}

The following two lemmas are central to the proof of \prettyref{thm:twoexpertsEM}. Let $\bA^\top=[\ba_1|\ldots|\ba_k] \in \reals^{d\times k}$ denote the matrix of regressor parameters whereas $\bW^\top=[\bw_1|\ldots|\bw_{k-1}] \in \reals^{d\times (k-1)}$ denote the matrix of gating parameters. With a slight change of notation, when $\bA=\bA^\ast$, we denote the EM operator $M(\bW)$ as either $M(\bW,\bA^\ast)$ or $M(\bw)$, introduced in \prettyref{sec:guaranteestwomoe}. For the general case, we simply denote it by $M(\bW,\bA)$.  In the following lemmas, we use the norm $\| \bA \| =\max_{i \in [k]}\| \bA_i^\top \|_2$ where $\bA \in \reals^{k \times d}$ is a matrix of regressors, similarly for any matrix of classifiers $\bW \in \reals^{(k-1)\times d}$.

\begin{lemma}[Contraction of the EM operator]
\label{lmm:contraction}
Under the assumptions of \prettyref{thm:twoexpertsEM}, we have that
\begin{align*}
\|M(\bW,\bA^\ast)-\bW^\ast \| \leq \kappa_\sigma \|\bW-\bW^\ast\|.
\end{align*}
Moreover, $\bW=\bW^\ast$ is a fixed point for $M(\bW,\bA^\ast)$.
\end{lemma}

\begin{lemma}[Robustness of the EM operator]
\label{lmm:robustness}
Let the matrix of regressors $\bA$ be such that $\max_{i \in [k]}\| \bA_i^\top - (\bA^\ast_i)^\top \|_2 = \sigma^2 \varepsilon_1 $. Then for any $\bW \in \Omega$, we have that
\begin{align*}
\|M(\bW,\bA)-M(\bW,\bA^\ast) \| \leq \kappa \varepsilon_1,
\end{align*}
where $\kappa$ is a constant depending on $g,k$ and $\sigma$. In particular, $\kappa \leq  (k-1)\frac{\sqrt{6(2+\sigma^2)}}{2}$ for $g=$linear, sigmoid and ReLU.
\end{lemma}

We are now ready to prove \prettyref{thm:twoexpertsEM}.

\begin{proof}
We first note that the EM iterates $\{ \bW_t \}_{t \geq 1}$ evolve according to
\begin{align*}
\bW_{t} =M(\bW_{t-1},\bA), \quad t \geq 1
\end{align*}
Thus
\begin{align*}
\norm{\bW_{t} -\bW^\ast} = \norm{M(\bW_{t-1},\bA) -\bW^\ast} &= \norm{M(\bW_{t-1},\bA) -M(\bW^\ast,\bA^\ast)} \\
& \leq \norm{M(\bW_{t-1},\bA) -M(\bW_{t-1},\bA^\ast)}+\norm{M(\bW_{t-1},\bA^\ast) -\bW^\ast}\\
& \leq k \varepsilon_1 + \kappa_\sigma \norm{\bW_{t-1}-\bW^\ast},
\end{align*}
where the last inequality follows from \prettyref{lmm:contraction} and \prettyref{lmm:robustness}. Recursively using the above inequality, we obtain that
\begin{align*}
\norm{\bW_{t} -\bW^\ast} \leq (\kappa_\sigma)^t \norm{\bW_0 -\bW^\ast} + \kappa\varepsilon_1(1+\kappa_\sigma+\ldots+\kappa_\sigma^{t-1}) \leq (\kappa_\sigma)^t \norm{\bW_0 -\bW^\ast}+\frac{\kappa\varepsilon_1}{1-\kappa_\sigma}.
\end{align*} 
\end{proof}

\subsection{Proof of \prettyref{thm:MoM_gating}}
\label{app:gatingspectral}
\begin{proof}
We are given that $(\ba_1,\ba_2)=(\ba_1^\ast,\ba_2^\ast)$. Denoting $\bw^\ast$ with $\bw$, from \prettyref{eq:twolinearmodel}, we have that
\begin{align}
\Expect [ y | \bx] &= f(\bw^\top \bx) \cdot \ba_1^\top \bx + (1-f(\bw^\top \bx)) \cdot \ba_2^\top \bx, \\
&=\ba_2^\top \bx + f(\bw^\top \bx) \cdot (\ba_1-\ba_2)^\top \bx.
\end{align}
Thus,
\begin{align*}
\frac{ \Expect [ y | \bx] - \ba_2^\top \bx}{(\ba_1-\ba_2)^\top \bx} &= f(\bw^\top \bx).
\end{align*}
Notice that in the above equation we have $(\ba_1-\ba_2)^\top \bx$ in the denominator. But this equals zero with zero probability whenever $\bx$ is generated from a continuous distribution; in our case $\bx$ is Gaussian. Thus we may write
\begin{align*}
\Expect \qth{ \pth{\frac{y-\ba_2^\top \bx}{(\ba_1-\ba_2)^\top \bx}} \cdot \bx  } \stackrel{\xmark} = \Expect \qth{\pth{\frac{\Expect[y|\bx]-\ba_2^\top \bx}{(\ba_1-\ba_2)^\top \bx}} \cdot \bx   } &= \Expect\qth{ f(\bw^\top \bx) \cdot \bx } \\
&= \Expect\qth{f'(\bw^\top \bx)} \cdot \bw \\
&= \Expect_{Z \sim \calN(0,1)} f'(\norm{\bw} Z) \cdot \bw\\
& \propto \bw.
\end{align*}
However, it turns out that the above chain of equalities does not hold. Surprisingly, the first equality, which essentially is the law of iterated expectations, is not valid in this case as $\frac{y-\ba_2^\top \bx}{(\ba_1-\ba_2)^\top \bx}$ is not integrable. To see this, notice that the model in \prettyref{eq:twolinearmodel} can also be written as
\begin{align*}
y \stackrel{(d)} = Z (\ba_1^\top \bx)+ (1-Z)(\ba_2^\top \bx) + \sigma N, \quad Z \sim \mathrm{Bern}(f(\bw^\top \bx)), N \sim \calN(0,1).
\end{align*}
Thus,
\begin{align*}
\mathrm{Ratio} \triangleq \frac{y-\ba_2^\top \bx}{(\ba_1-\ba_2)^\top \bx} \stackrel{(d)} = Z + \frac{\sigma N}{(\ba_1-\ba_2)^\top \bx}.
\end{align*}
Since $Z$ is independent of $N$ and $\frac{N}{(\ba_1-\ba_2)^\top \bx}$ is a Cauchy random variable, it follows that the random variable $\mathrm{Ratio}$ is not integrable. To deal with the non-integrability of $\mathrm{Ratio}$, we look at its conditional cdf, given by 
 \begin{align*}
\prob{\mathrm{Ratio} \leq z|\bx} = f(\bw^\top\bx) \Phi \pth{(z-1) \frac{|\Delta_x|}{\sigma}}+(1-f(\bw^\top\bx)) \Phi \pth{z \frac{|\Delta_x|}{\sigma}}, \quad \Delta_x=(\ba_1-\ba_2)^\top \bx,
\end{align*}
where $\Phi(\cdot)$ is the standard Gaussian cdf. Substituting $z=0.5$ and using the fact that $\Phi(z)+\Phi(-z)=1$, we obtain
\begin{align*}
\prob{\mathrm{Ratio} \leq 0.5|\bx} &= f(\bw^\top\bx) \Phi \pth{ -\frac{|\Delta_x|}{2\sigma}}+(1-f(\bw^\top\bx)) \Phi \pth{\frac{|\Delta_x|}{2\sigma}}\\
&=  \Phi \pth{\frac{|(\ba_1-\ba_2)^\top \bx|}{2\sigma}} + f(\bw^\top \bx) \pth{1-2\Phi \pth{\frac{|(\ba_1-\ba_2)^\top \bx|}{2\sigma}}}.
\end{align*}
Since $\Phi \pth{\frac{|(\ba_1-\ba_2)^\top \bx|}{2\sigma}} $ is a symmetric function in $\bx$ its first moment with $\bx$ equals zero. Furthermore, if we assume that $\bw$ is orthogonal to $\ba_1$ and $\ba_2$, we have
\begin{align*}
\Expect\qth{\mathds{1}\sth{\mathrm{Ratio} \leq 0.5}\cdot \bx} &= \Expect \qth{\prob{\mathrm{Ratio} \leq 0.5|\bx} \cdot \bx} \\
&=\Expect\qth{f(\bw^\top \bx) \pth{1-2\Phi \pth{\frac{|(\ba_1-\ba_2)^\top \bx|}{2\sigma}}} \cdot \bx}\\
&=\Expect[f'(\bw^\top\bx)]\cdot \Expect\pth{1-2\Phi \pth{\frac{|(\ba_1-\ba_2)^\top \bx|}{2\sigma}}} \cdot \bw +\\
& \hspace{4em} \Expect[f(\bw^\top \bx)]\cdot \underbrace{\Expect \qth{\nabla_{\bx} \pth{1-2\Phi \pth{\frac{|(\ba_1-\ba_2)^\top \bx|}{2\sigma}}}}}_{=0, \text{ since derivative of a even function is odd} } \\
&=\Expect[f'(\bw^\top\bx)]\cdot \Expect\pth{1-2\Phi \pth{\frac{|(\ba_1-\ba_2)^\top \bx|}{2\sigma}}} \cdot \bw \\
& \propto \bw.
\end{align*}
Thus, if $\norm{\bw}=1$, we have that 
$$\frac{\Expect\qth{\mathds{1}\sth{\mathrm{Ratio} \leq 0.5}\cdot \bx}}{\norm{\Expect\qth{\mathds{1}\sth{\mathrm{Ratio} \leq 0.5}\cdot \bx}}}=\bw.
$$ 
In the finite sample regime, $\Expect\qth{\mathds{1}\sth{\mathrm{Ratio} \leq 0.5}\cdot \bx}$ can be estimated from samples using the empirical moments and its normalized version will be an estimate of $\bw$.
\end{proof}

\section{Proof of \prettyref{lmm:robustness}}
\label{app:twoexpertsperturbation}
We need the following lemma which establishes the stability of the minimizers for strongly convex functions under Lipschitz perturbations.
\begin{lemma}
\label{lmm:stabilityconvexity}
Suppose $\Omega \subseteq \reals^d$ is a closed convex subset, $f: \Omega \to \reals$ is a $\lambda$-strongly convex function for some $\lambda>0$ and $B$ is an $L$-Lipschitz continuous function on $\Omega$. Let $\bw_f=\argmin_{\bw \in \Omega}f(\bw)$ and $\bw_{f+B}=\argmin_{\bw \in \Omega}f(\bw)+B(\bw)$. Then 
\begin{align*}
\norm{\bw_f -\bw_{f+B}} \leq \frac{L}{\lambda}.
\end{align*}
\end{lemma}
\begin{proof}
Let $\bw' \in \Omega$ be such that $\norm{\bw'-\bw_f} > \frac{L}{\lambda}$. Let $\bw_\alpha = \alpha \bw_f+(1-\alpha) \bw'$ for $0<\alpha<1$. From the fact that $\bw_f$ is the minimizer of $f$ on $\Omega$ and that $f$ is strongly convex, we have that
\begin{align*}
f(\bw') \geq f(\bw_f) + \frac{\lambda \norm{\bw'-\bw_f}^2}{2}.
\end{align*}
Furthermore, the strong-convexity of $f$ implies that
\begin{align}
f(\bw_\alpha) & \leq \alpha f(\bw_f)+(1-\alpha) f(\bw')- \frac{\alpha(1-\alpha)\lambda}{2}\norm{\bw'-\bw_f}^2 \nonumber \\
&= f(\bw') + \alpha(f(\bw_f)-f(\bw'))- \frac{\alpha(1-\alpha)\lambda}{2}\norm{\bw'-\bw_f}^2 \nonumber \\
& \leq f(\bw')- \alpha \frac{\lambda \norm{\bw'-\bw_f}^2}{2}- \frac{\alpha(1-\alpha)\lambda}{2}\norm{\bw'-\bw_f}^2 \nonumber \\
&=f(\bw') - \lambda \alpha\pth{1-\frac{\alpha}{2}}\norm{\bw'-\bw_f}^2  \label{eq:first}
\end{align} 
Since $B$ is $L$-Lipschitz, we have
\begin{align}
B(\bw_\alpha) \leq B(\bw')+ L \alpha \norm{\bw'-\bw_f}. 
\label{eq:second}
\end{align}
Adding \prettyref{eq:first} and \prettyref{eq:second}, we get
\begin{align*}
f(\bw_\alpha)+B(\bw_\alpha) & \leq f(\bw')+B(\bw')+L\alpha \norm{\bw'-\bw_f}- \lambda \alpha\pth{1-\frac{\alpha}{2}}\norm{\bw'-\bw_f}^2 \\
&=f(\bw')+B(\bw') + \alpha \lambda \norm{\bw'-\bw_f} \pth{\frac{L}{\lambda}-\pth{1-\frac{\alpha}{2}}\norm{\bw'-\bw_f}}
\end{align*}
By the assumption that $\norm{\bw'-\bw_f} > \frac{L}{\lambda}$, the term $\frac{L}{\lambda}-\pth{1-\frac{\alpha}{2}}\norm{\bw'-\bw_f}$ will be negative for sufficiently small $\alpha$. This in turn implies that $f(\bw_\alpha)+B(\bw_\alpha) < f(\bw')+B(\bw')$ for such $\alpha$. Consequently $\bw'$ is not a minimizer of $f+B$ for any $\bw'$ such that $\norm{\bw'-\bw_f} > \frac{L}{\lambda}$. The conclusion follows.

\end{proof}

We are now ready to prove \prettyref{lmm:robustness}.
Fix any $\bW \in \Omega$ and let $\bA =\begin{bmatrix}
\ba_1^\top \\ \ldots \\ \ba_k^\top
\end{bmatrix} \in \reals^{k \times d} $ be such that $\max_{i \in [k]}\norm{\ba_i-\ba_i^\ast}_2 =\sigma^2 \varepsilon_1 $ for some $\varepsilon_1>0$. Let 
\begin{align*}
\bW'=M(\bW,\bA), \quad (\bW')^\ast = M(\bW,\bA^\ast),
\end{align*}
where,
\begin{align*}
M(\bW,\bA) = \arg \max_{\bW' \in \Omega} Q(\bW'|\bW,\bA),
\end{align*}
and, 
\begin{align*}
Q(\bW'|\bW,\bA)= \Expect \qth{ \sum_{i \in [k-1]} p^{(i)}(\bW,\bA) ((\bW'_i)^\top \bx) - \log \pth{1+\sum_{i \in [k-1]} e^{(\bW'_i)^\top \bx}} }.
\end{align*}
Here $p^{(i)}(\bA,\bW) \define \frac{p_i(\bx)N_i}{\sum_{i \in [k]}p_i(\bx)N_i}$ denotes the posterior probability of choosing the $i^{\mathrm{th}}$ expert, where
\begin{align*}
p_i(\bx) = \frac{e^{\bw_i^\top \bx}}{1+\sum_{k \in [k-1]}e^{\bw_j^\top \bx}}, \quad N_i \define \calN(y|g(\ba_i^\top \bx),\sigma^2), \quad N_i^\ast= \calN(y|g((\ba_i^\ast)^\top \bx),\sigma^2).
\end{align*}

Since both $Q(\cdot|\bW,\bA)$ and $Q(\cdot|\bW,\bA^\ast)$ are strongly concave functions over $\Omega$ with some strong-concavity parameter $\lambda$, \prettyref{lmm:stabilityconvexity} implies that
\begin{align*}
\norm{M(\bW,\bA) - M(\bW,\bA^\ast)} \leq \frac{L}{\lambda},
\end{align*}
where $L$ is the Lipschitz-constant for the function $l(\cdot) \triangleq Q(\cdot|\bW,\bA)-Q(\cdot|\bW,\bA^\ast)$. We have that
\begin{align*}
l(\bW') =\sum_{i \in [k-1]} \Expect[ (p^{(i)}(\bW,\bA)-p^{(i)}(\bW,\bA^\ast)(\bW'_i)^\top \bx) ]
\end{align*}
Without loss of generality let $i=1$. Since $l(\cdot)$ is linear in $\bW'$, it suffices to show for each $i$ that
\begin{align*}
\norm{ \Expect[ (p^{(1)}(\bW,\bA)-p^{(1)}(\bW,\bA^\ast)\bx ]} \leq L,
\end{align*}
We show that $L = \kappa \varepsilon_1$, or equivalently,
\begin{align*}
\norm{ \Expect[ (p^{(1)}(\bW,\bA)-p^{(1)}(\bW,\bA^\ast)\bx ]} \leq \kappa \varepsilon_1,
\end{align*}

Let
\begin{align*}
\bA_t= \bA^\ast+ t \Delta, \quad \Delta=\bA-\bA^\ast \in \reals^{k \times d}.
\end{align*}
By hypothesis, we have that $\norm{\Delta_i}_2 \leq \sigma^2 \varepsilon_1$ for all $i \in [k]$. Thus in order to show that 
\begin{align*}
\norm{\Expect[(p^{(1)}(\bA,\bW)-p^{(1)}(\bA^\ast,\bW)  )\bx]}_2 \leq \kappa \varepsilon_1,
\end{align*}
it suffices to show that
\begin{align*}
\inner{\Expect[(p^{(1)}(\bA,\bW)-p^{(1)}(\bA^\ast,\bW)  )\bx]}{\tilde{\Delta}} \leq \kappa \norm{\Delta/\sigma^2}_2 \| \tilde{\Delta}\|_2, \quad \text{ for all } \tilde{\Delta} \in \reals^d.
\end{align*}
Or equivalently,
\begin{align*}
\Expect[ (p^{(1)}(\bA,\bW)-p^{(1)}(\bA^\ast,\bW)) \inner{\bx}{\tilde{\Delta}}] \leq \kappa \norm{\Delta/\sigma^2}_2 \| \tilde{\Delta}\|_2.
\end{align*}
We can rewrite the difference of the posteriors as
\begin{align}
p^{(1)}(\bA,\bW)-p^{(1)}(\bA^\ast,\bW) = \int_{0}^1 \frac{d}{dt}p^{(1)}(\bA^\ast+t\Delta,\bW) dt=\sum_{i \in [k]} \int_0^1 \inner{\nabla_{\ba_i} p^{(1)}(\bA_t,\bW) }{\Delta_i} dt.
\label{eq:}
\end{align}
Since $N_i=\calN(y|g(\ba_i^\top \bx),\sigma^2)=\frac{1}{\sqrt{2 \pi \sigma^2}}e^{-(y-g(\ba_1^\top \bx))^2/2\sigma^2}$, we have that
\begin{align*}
\nabla_{\ba_i} N_i = N_i \pth{\frac{y-g(\ba_i^\top \bx)}{\sigma^2}}g'(\ba_i^\top \bx).
\end{align*}
Thus,
\begin{align*}
\nabla_{\ba_i} p^{(1)}(\bA_t,\bW) &= \nabla_{\ba_i} \pth{\frac{p_1(\bx)N_1}{\sum_{i \in [k]}p_i(\bx)N_i}}\\
&= \begin{cases}
\frac{(\sum_{i \neq 1}p_i(\bx)N_i)p_1(\bx)N_1}{(\sum_{i}p_i(\bx)N_i)^2}\pth{\frac{y-g(\ba_1^\top \bx)}{\sigma^2}}g'(\ba_1^\top \bx)\bx,  & \text{ if } i=1 \\
\frac{-p_i(\bx)p_1(\bx) N_i N_1}{(\sum_{i}p_i(\bx)N_i)^2}\pth{\frac{y-g(\ba_i^\top \bx)}{\sigma^2}}g'(\ba_i^\top \bx)\bx,  & \text{ if }i \neq 1
\end{cases}
\end{align*}
Hence,
\begin{align}
\Expect[ (p^{(1)}(\bA,\bW)-p^{(1)}(\bA^\ast,\bW)) \inner{\bx}{\tilde{\Delta}}] &= \sum_{i \in [k]}\int_0^1 \Expect[\inner{\nabla_{\ba_i} p^{(1)}(\bA_t,\bW) }{\Delta_i}\inner{\bx}{\tilde{\Delta}}  ] dt \\
&=\int_0^1 \Expect \qth{ \frac{(\sum_{i \neq 1}p_i(\bx)N_i)p_1(\bx)N_1}{(\sum_{i}p_i(\bx)N_i)^2}\pth{\frac{y-g(\ba_1^\top \bx)}{\sigma^2}}g'(\ba_1^\top \bx) \inner{\bx}{\Delta_1}  \inner{\bx}{\tilde{\Delta}}} dt \\
&+\sum_{i \neq 1}\int_0^1 \Expect \qth{ \frac{-p_i(\bx)p_1(\bx) N_i N_1}{(\sum_{i}p_i(\bx)N_i)^2}\pth{\frac{y-g(\ba_i^\top \bx)}{\sigma^2}}g'(\ba_i^\top \bx) \inner{\bx}{\Delta_i}  \inner{\bx}{\tilde{\Delta}}} dt,
\label{eq:diffderivative}
\end{align}
where we denoted $(\ba_i)_t$ by $\ba_i$ in the integrals above(with a slight abuse of notation) for the sake of notational simplicity.
For any $i \neq 1$, we have that
\begin{align*}
\left|\frac{-p_i(\bx)p_1(\bx) N_i N_1}{(\sum_{i}p_i(\bx)N_i)^2}\pth{\frac{y-g(\ba_i^\top \bx)}{\sigma^2}}g'(\ba_i^\top \bx) \inner{\bx}{\Delta_i}  \inner{\bx}{\tilde{\Delta}} \right| \\
\leq \frac{p_i(\bx)p_1(\bx) N_i N_1}{(p_1(\bx)N_1+p_i(\bx)N_i)^2}|(y-g(\ba_i^\top \bx))g'(\ba_i^\top \bx)\inner{\bx}{\Delta_i/\sigma^2}  \inner{\bx}{\tilde{\Delta}}|
\end{align*}
For $g=$linear, sigmoid and ReLU, we have that $|g'(\cdot)|\leq 1$. Moreover, $\frac{p_i(\bx)p_1(\bx) N_i N_1}{(p_1(\bx)N_1+p_i(\bx)N_i)^2} \leq 1/4$. Thus we have
\begin{align*}
\frac{p_i(\bx)p_1(\bx) N_i N_1}{(p_1(\bx)N_1+p_i(\bx)N_i)^2}|(y-g(\ba_i^\top \bx))g'(\ba_i^\top \bx)\inner{\bx}{\Delta_i/\sigma^2}  \inner{\bx}{\tilde{\Delta}}| \leq \frac{1}{4}|y-g(\ba_i^\top \bx)||\inner{\bx}{\Delta_i/\sigma^2}  \inner{\bx}{\tilde{\Delta}}|.
\end{align*}
We thus get
\begin{align}
\Expect \qth{ \frac{-p_i(\bx)p_1(\bx) N_i N_1}{(\sum_{i}p_i(\bx)N_i)^2}\pth{\frac{y-g(\ba_i^\top \bx)}{\sigma^2}}g'(\ba_i^\top \bx) \inner{\bx}{\Delta_i}  \inner{\bx}{\tilde{\Delta}}} &\leq \frac{1}{4}\Expect[|y-g(\ba_i^\top \bx)||\inner{\bx}{\Delta_i/\sigma^2}  \inner{\bx}{\tilde{\Delta}}|]\\
&\leq \frac{1}{4}\sqrt{\Expect[(y-g(\ba_i^\top \bx))^2]\Expect[\inner{\bx}{\Delta_i/\sigma^2}^2  \inner{\bx}{\tilde{\Delta}}^2]}\\
& \leq \frac{\sqrt{3}}{4} \sqrt{\Expect[(y-g(\ba_i^\top \bx))^2]} \|\Delta_i/\sigma^2\|_2 \|\tilde{\Delta}\|_2
\label{eq:ibound}
\end{align}
Now it remains to bound $ \sqrt{\Expect[(y-g(\ba_i^\top \bx))^2]}$. Since $\norm{\ba_i}_2 \leq 1$, one can show that $\Expect[g(\ba_i^\top \bx)^2]\leq 1$ for the given choice of non-linearities for $g$. Also, we have that
\begin{align*}
\Expect[y^2]=\Expect[\Expect[y^2|\bx]]=\Expect[\sum_{i\in [k]}p_i^\ast(\bx)g(\inner{\ba_i^\ast}{\bx})^2 + \sigma^2]=\Expect[\sum_{i \in [k]}p_i^\ast(\bx)]\Expect[g(\inner{\ba_1^\ast}{\bx})^2]+\sigma^2 \leq 1+ \sigma^2,
\end{align*}
where we used the following facts: (i) $\inner{\ba_i^\ast}{ \bx}$ is independent of the random variable $p_i^\ast(\bx)$ for each $i \in [k]$, (ii) $\inner{\ba_i^\ast}{ \bx} \stackrel{(d)}= \inner{\ba_1^\ast}{ \bx}$ and (iii) $\Expect[g(\inner{\ba_1^\ast}{\bx})^2]\leq 1$. Since $\Expect[(y-g(\ba_i^\top \bx))^2] \leq 2 \Expect[y^2]+\Expect[g(\ba_i^\top \bx)^2]$, after substituting these bounds in \prettyref{eq:ibound}, we get
\begin{align*}
\Expect \qth{ \frac{-p_i(\bx)p_1(\bx) N_i N_1}{(\sum_{i}p_i(\bx)N_i)^2}\pth{\frac{y-g(\ba_i^\top \bx)}{\sigma^2}}g'(\ba_i^\top \bx) \inner{\bx}{\Delta_i}  \inner{\bx}{\tilde{\Delta}}} \leq \frac{\sqrt{6(2+\sigma^2)}}{4}\|\Delta_i/\sigma^2\|_2 \|\tilde{\Delta}\|_2.
\end{align*}
Similarly,
\begin{align*}
\Expect \qth{ \frac{p_i(\bx)N_ip_1(\bx)N_1}{(\sum_{i}p_i(\bx)N_i)^2}\pth{\frac{y-g(\ba_1^\top \bx)}{\sigma^2}}g'(\ba_1^\top \bx) \inner{\bx}{\Delta_1}  \inner{\bx}{\tilde{\Delta}}} \leq \frac{\sqrt{6(2+\sigma^2)}}{4}\|\Delta/\sigma^2\|_2 \|\tilde{\Delta}\|_2.
\end{align*}
Substituting the above two inequalities in \prettyref{eq:diffderivative}, we obtain that
\begin{align*}
\Expect[ (p^{(1)}(\bA,\bW)-p^{(1)}(\bA^\ast,\bW)) \inner{\bx}{\tilde{\Delta}}] & \leq 2(k-1)\frac{\sqrt{6(2+\sigma^2)}}{4}\|\Delta_1/\sigma^2\|_2 \|\tilde{\Delta}\|_2.
\end{align*}
Defining $\kappa \define (k-1)\frac{\sqrt{6(2+\sigma^2)}}{2}$ and using the fact that $\norm{\Delta/\sigma^2}_2 \leq \varepsilon_1$, we thus obtain
\begin{align*}
\norm{\Expect[(p^{(1)}(\bA,\bW)-p^{(1)}(\bA^\ast,\bW)  )\bx]}_2 \leq \kappa \varepsilon_1.
\end{align*}

\section{Proof of \prettyref{lmm:contraction}}
\label{app:proofcontraction}

\subsection{Proof for $k=2$}
\begin{proof}

We first prove the lemma for $k=2$. We show that the assumptions in \prettyref{app:reviewem} hold \emph{globally} in our setting yielding a geometric convergence. Here we simply denote $M(\bW,\bA^\ast)$ as $M(\bw)$ dropping the explicit dependence on $\bA^\ast$. Recall that
\begin{align*}
Q(\bw|\bw_t)=\Expect_{p_{\bw^\ast}(\bx,y)} \qth{p_1(\bx,y,\bw_t) \cdot (\bw^\top\bx) -\log(1+e^{\bw^\top \bx})}, 
\end{align*}
where 
\begin{align}
p_1(\bx,y,\bw_t)=\frac{f(\bw_t^\top \bx) \calN(y|g(\ba_1^\top \bx), \sigma^2)}{f(\bw^\top \bx) \calN(y|g(\ba_1^\top \bx), \sigma^2)+(1-f(\bw^\top \bx) )\calN(y|g(\ba_2^\top \bx),\sigma^2)}.
\label{eq:posteriorappendix}
\end{align}
For simplicity we drop the subscript in the above expectation with respect to the distribution ${p_{\bw^\ast}(\bx,y)}$. Now we verify each of the assumptions.
\begin{itemize}
\item Convexity of $\Omega$ easily follows from its definition.
\item We have that
\begin{align*}
Q(\bw|\bw^\ast)&=\Expect \qth{p_1(\bx,y,\bw^\ast) \cdot (\bw^\top\bx) -\log(1+e^{\bw^\top \bx)}}.
\end{align*}
Note that the strong-concavity of $Q(\cdot|\bw^\ast)$ is equivalent to the strong-convexity of $-Q(\cdot|\bw^\ast)$. 
Denoting the sigmoid function by $f$, we have that for all $\bw \in \Omega$,
\begin{align}
- \nabla^{2} Q(\bw|\bw^\ast)& = \Expect \qth{f'(\bw^\top \bx) \cdot \bx \bx^\top}, \nonumber \\
&\stackrel{(\text{Stein's lemma})}= \Expect \qth{f'''(\bw^\top \bx)} \cdot \bw \bw^\top + \Expect[f'(\bw^\top \bx)] \cdot I  \nonumber \\
&=\Expect[f'''(\norm{\bw} Z)] \cdot  \bw \bw^\top + \Expect[f'(\norm{\bw} Z)] \cdot I, \quad Z \sim \calN(0,1) \nonumber \\
& \stackrel{(a)}\succcurlyeq \inf_{0 \leq \alpha \leq 1} \min \sth{\Expect[f'(\alpha Z)],\Expect[f'(\alpha Z)] + \alpha^2 \Expect[f'''(\alpha Z)]} \cdot I \nonumber\\
&= \underbrace{0.14}_{\lambda} \cdot  I \label{eq:strongconvexity}
\end{align}
where $(a)$ follows from finding the two possible eigenvalues of the positive-definite matrix in the previous step and considering the minimum among them to ensure strong-convexity. Here the value of $\lambda$ is found numerically to be approximately around $0.1442$.
\item For any $\bw,\bw_t \in \Omega$,
\begin{align*}
\nabla Q(\bw|\bw_t)&=\Expect \qth{p_1(\bx,y,\bw_t) \cdot \bx  -f(\bw^\top \bx) \cdot \bx}.
\end{align*}
Thus,
\begin{align*}
\norm{\nabla Q(M(\bw)|\bw^\ast) - \nabla Q(M(\bw)|\bw)} &=\norm{\Expect \qth{ \pth{p_1(\bx,y,\bw_t)-p_1(\bx,y,\bw^\ast} \cdot \bx }} \stackrel{(a)} \leq \gamma_\sigma \norm{\bw-\bw^\ast},
\end{align*}
where we want to prove in $(a)$ that $\gamma_\sigma$ is smaller than $0.14$ for all $\bw \in \Omega$. Intuitively, this means that the posterior probability in \prettyref{eq:posteriorappendix} is smooth with respect to the parameter $\bw$. We will now show that this can be achieved in the high-SNR regime when $\sigma$ is sufficiently small. This will ensure that $\kappa_\sigma \triangleq \frac{\gamma_\sigma}{\lambda} < 1$. In particular, the value of $\gamma_\sigma$ is dimension-independent and depends only on the choice of the non-linearity $g$.
\end{itemize}
To prove that 
\begin{align*}
\norm{\Expect \qth{ \pth{p_{1}(\bx,y,\bw)-p_{1}(\bx,y,\bw^\ast)} \cdot \bx }} \leq \gamma \norm{\bw-\bw^\ast}=\gamma \norm{\Delta},
\end{align*}
it suffices to show
\begin{align*}
\inner{\Expect \qth{ \pth{p_{1}(\bx,y,\bw)-p_{1}(\bx,y,\bw^\ast)} \cdot \bx }}{\tilde{\Delta}} &\leq \gamma \norm{\Delta}\|\tilde{\Delta}\|, \quad \forall \tilde{\Delta} \in \reals^d.
\end{align*}
Or equivalently,
\begin{align*}
\Expect \qth{ \pth{p_{1}(\bx,y,\bw)-p_{1}(\bx,y,\bw^\ast)} \inner{\bx}{\tilde{\Delta}} } &\leq \gamma \norm{\Delta}\|\tilde{\Delta}\|.
\end{align*}
 Let $\Delta \triangleq \bw-\bw^\ast$ and $f(u) \triangleq p_1(\bx,y,\bw_u)$ where $\bw_u=\bw^\ast+u \Delta, u \in [0,1]$. Thus $f(1)=p_1(\bx,y,\bw)$ and $f(0)=p_1(\bx,y,\bw^\ast)$. So we get
\begin{align*}
p_{1}(\bx,y,\bw)-p_{1}(\bx,y,\bw^\ast) = f(1)-f(0) = \int_0^1 f'(u) du=\int_0^1 \inner{\nabla p_1(\bx,y,\bw_u)}{\Delta} du,
\end{align*}
where the gradient is evaluated with respect to $\bw_u$. Differentiating \prettyref{eq:posteriorappendix} with respect to $\bw$, we get that
\begin{align*}
\nabla_{\bw} p_1(\bx,y,\bw)&=\frac{f(\bw^\top \bx)  (1-f(\bw^\top \bx) )\calN(y|g(\ba_1^\top \bx), \sigma^2)\calN(y|g(\ba_2^\top \bx),\sigma^2)}{(f(\bw^\top \bx) \calN(y|g(\ba_1^\top \bx), \sigma^2)+(1-f(\bw^\top \bx) )\calN(y|g(\ba_2^\top \bx),\sigma^2))^2} \cdot \bx \\
&\triangleq R(\bx,y,\bw,\sigma) \cdot \bx.
\end{align*}
Thus,
\begin{align*}
\Expect \qth{ \pth{p_1(\bx,y,\bw)-p_1(\bx,y,\bw^\ast)} \inner{\bx}{\tilde{\Delta}} } 
&= \Expect \qth{ \pth{ \int_0^1 R(\bx,y,\bw_u ,\sigma)  \inner{\bx}{\Delta} du}  \inner{\bx}{\tilde{\Delta}}} \\
&= \int_0^1 \Expect \qth{R(\bx,y,\bw_u ,\sigma) \inner{\bx}{\Delta}  \inner{\bx}{\tilde{\Delta}} } du \\
& \leq {\pth{\int_0^1 \sqrt{\Expect[R(\bx,y,\bw_u ,\sigma)^2]} du }}{\sqrt{\Expect \qth{\inner{\bx}{\Delta}^2 \inner{\bx}{\tilde{\Delta}}^2}}} \\
&\leq \underbrace{\sqrt{3}\pth{\int_0^1 \sqrt{\Expect[R(\bx,y,\bw_u ,\sigma)^2]} du }}_{\gamma_\sigma}\norm{\Delta} \| \tilde{\Delta}\|\\
&= \gamma_\sigma \norm{\Delta} \| \tilde{\Delta}\|,
\end{align*}
where the last inequality follows from Lemma $5$ of \cite{bala17}. Our goal is to now prove that $\gamma_\sigma \rightarrow 0$ as $\sigma \rightarrow 0$. First observe that
\begin{align*}
R(\bx,y,\bw,\sigma) &= \frac{f(\bw^\top \bx)  (1-f(\bw^\top \bx) e^{-(y-g(\ba_1^\top \bx))/2\sigma^2} e^{-(y-g(\ba_1^\top \bx))/2\sigma^2} }{(f(\bw^\top \bx) e^{-(y-g(\ba_1^\top \bx))/2\sigma^2}+(1-f(\bw^\top \bx) )e^{-(y-g(\ba_2^\top \bx))/2\sigma^2})^2} \leq \frac{1}{4} (\text{ since } \frac{ab}{(a+b)^2} \leq 1/4 ) \\
& = \frac{f(1-f) e^{\frac{(y-g(\ba_1^\top \bx))^2- (y-g(\ba_2^\top \bx))^2}{2\sigma^2}}}{\pth{f+(1-f)e^{\frac{(y-g(\ba_1^\top \bx))^2- (y-g(\ba_2^\top \bx))^2}{2\sigma^2}}}^2} \rightarrow 0 \text{ as } \sigma  \rightarrow 0,
\end{align*}
where the key observation is that irrespective of the sign of $(y-g(\ba_1^\top \bx))^2- (y-g(\ba_2^\top \bx))^2$, the ratio still goes to zero and hence by dominated convergence theorem $\Expect[R(\bx,y,\bw_u,\sigma)^2] \rightarrow 0$ for each $u \in [0,1]$. Now we show that this convergence is uniform in $u$ and thus $\gamma_\sigma \rightarrow 0$. For simplicity, define
\begin{align}
\Delta_1 \triangleq (y-g(\ba_1^\top \bx))^2, \quad \Delta_2 \triangleq (y-g(\ba_2^\top \bx))^2 \text{ and } \sigma=\frac{1}{n}. \label{eq:needagain}
\end{align}
Thus,
\begin{align}
R(\bx,y,\bw_u,\sigma) &= \frac{f(1-f)e^{\frac{n^2}{2}(\Delta_1-\Delta_2)}}{\pth{f+(1-f)e^{\frac{n^2}{2}(\Delta_1-\Delta_2)}}^2}\\
& \leq \frac{f(1-f)e^{\frac{n^2}{2}(\Delta_1-\Delta_2)}}{\pth{(1-f)e^{\frac{n^2}{2}(\Delta_1-\Delta_2)}}^2} = \frac{f}{1-f} e^{-\frac{n^2}{2}(\Delta_1-\Delta_2)}.
\end{align} 
Similarly, 
\begin{align}
R(\bx,y,\bw_u,\sigma) & \leq \frac{1-f}{f} e^{-\frac{n^2}{2}(\Delta_2-\Delta_1)}.
\end{align}
Thus, we get
\begin{align}
R(\bx,y,\bw_u,\sigma) & \leq \max\pth{\frac{1-f}{f},\frac{f}{1-f}} e^{-\frac{n^2}{2}(|\Delta_1-\Delta_2|)}.
\end{align}
Hence
\begin{align}
\frac{\gamma_\sigma}{\sqrt{3}} &= \int_0^1 \sqrt{\Expect[\mathrm{Ratio}(\bx,y,\bw_u,\sigma)^2]} du \\
&\leq \int_0^1 \sqrt{\Expect \qth{ \max \pth{\frac{1-f}{f},\frac{f}{1-f}}^2 e^{-n^2 |\Delta_1-\Delta_2|} }   } du \\
& \leq \int_0^1 \sqrt{\Expect \qth{ \pth{\frac{1-f}{f}}^2 e^{-n^2 |\Delta_1-\Delta_2|} +\pth{\frac{f}{1-f}}^2 e^{-n^2 |\Delta_1-\Delta_2|} }   } du \\
& = \int_0^1 \sqrt{2 \Expect \qth{e^{2\bw_u ^\top \bx} e^{-n^2 |\Delta_1-\Delta_2|}}} du \\
& \leq \int_0^1 \sqrt{2 \sqrt{\Expect[e^{4\bw_u^\top \bx}] \Expect[e^{-2n^2|\Delta_1-\Delta_2|}]}} du \\
& \stackrel{(a)}\leq \sqrt{2 e^4 \sqrt{\Expect[e^{-2n^2 |\Delta_1-\Delta_2|}]}},
\end{align}
where $(a)$ follows from the fact $\norm{\bw_u} \leq 1$ and $\Expect[e^{4\bw_u^\top \bx}]=e^{8\norm{\bw_u}^2} \leq e^8$, for each $u \in [0,1]$. Now we analyze the convergence rate of the last term $\Expect[e^{-2n^2 |\Delta_1-\Delta_2|}]$ for the case of linear regression, \ie $g(z)=z$. Notice that for the two-mixtures, we have
\begin{align}
y \stackrel{(d)} = Z (\ba_1^\top \bx)+ (1-Z) \ba_2^\top \bx + \sigma
N =Z (\ba_1^\top \bx)+ (1-Z) \ba_2^\top \bx + \frac{N}{n}, \quad Z|\bx \sim \mathrm{Bern}(f(\bw_\ast^\top \bx)).
\end{align}
Thus,
\begin{align}
\Delta_1-\Delta_2 & \stackrel{(d)} = (y-\ba_1^\top \bx)^2 - (y-\ba_2^\top \bx)^2 \\
 & = (\ba_1^\top \bx-\ba_2^\top \bx)^2 (1-2Z) + \frac{2N}{n}(\ba_2^\top \bx-\ba_1^\top \bx) \\
 & = \inner{\bx}{\bv}^2 (1-2Z)+ \frac{2N}{n}\inner{\bx}{\bv}, \quad \bv=\ba_1-\ba_2. 
\end{align}
Since $Z$ can equal either $0$ or $1$, we have
\begin{align}
\gamma_\sigma &\leq \sqrt{3} \sqrt{2e^4}\pth{\Expect[e^{-2n^2\left | \inner{\bx}{\bv}^2 (1-2Z)+ \frac{2N}{n}\inner{\bx}{\bv} \right|} ] }^{1/4}\\
& \leq \sqrt{6e^4} \pth{\Expect \qth{\max \pth{e^{-2n^2\left | \inner{\bx}{\bv}^2 + \frac{2N}{n}\inner{\bx}{\bv} \right|}, e^{-2n^2\left | -\inner{\bx}{\bv}^2+ \frac{2N}{n}\inner{\bx}{\bv} \right|}  }} }^{1/4} \\
& \leq  \sqrt{6\sqrt{2} e^4} \pth{\Expect \qth{e^{-2n^2|\inner{\bx}{\bv}^2+ \frac{2N}{n}\inner{\bx}{\bv}|}} }^{1/4} \\
& = \sqrt{6 \sqrt{2} e^4} \pth{ \Expect \qth{ e^{-2n^2|Z^2+\frac{2ZN}{n}|} }}^{1/4},\quad  Z \sim \calN(0,\norm{\ba_1-\ba_2}), N \sim  \calN(0,1).\\
&= O \pth{ \sqrt{6 \sqrt{2} e^4}\pth{ \Expect[e^{-2n^2Z^2}]}^{1/4}}\\
&= \sqrt{6 \sqrt{2} e^4} \pth{\sqrt{\frac{1}{4n^2\norm{\ba_1-\ba_2}^2+1}}}^{1/4}\\
&=O\pth{\frac{1}{\pth{n\norm{\ba_1-\ba_2}}^{1/4}}}\\
&=O \pth{\pth{\frac{\sigma}{\norm{\ba_1-\ba_2}}}^{1/4}}.
\end{align}
\end{proof}

\subsection{Proof for general $k$}
\begin{proof}
The proof strategy for general $k$ is similar. First let $\varepsilon_1=0$. Our task is to show that the assumptions of \prettyref{app:reviewem} hold globally in our setting. The domain $\Omega$ is clearly convex since
$$
\Omega = \sth{\bw=(\bw_1,\ldots,\bw_{k-1}): \norm{\bw_i} \leq 1, \forall i \in [k-1]}.
$$
Now we verify \prettyref{assump:A2}. The function $Q(.|\bw_t)$ is given by
\begin{align*}
Q(\bw|\bw_t)= \Expect \qth{ \sum_{i \in [k-1]} p_{\bw_t}^{(i)} (\bw_i^\top \bx) - \log \pth{1+\sum_{i \in [k-1]} e^{\bw_i^\top \bx}} },
\end{align*}
where $p_{\bw_t}^{(i)} \triangleq \prob{z=i|\bx,y,\bw_t}$ corresponds to the posterior probability for the $i^{\text{th}}$ expert, given by
\begin{align*}
p_{\bw_t}^{(i)} = \frac{ p_{i,t}(\bx) \calN(y|g(\ba_i^\top \bx),\sigma^2)}{\sum_{j \in [k] } p_{j,t}(\bx) \calN(y|g(\ba_j^\top \bx),\sigma^2)}, \quad p_{i,t}(\bx)= \frac{e^{(\bw_t)_i^\top \bx}}{1+\sum_{ j \in [k-1]}e^{(\bw_t)_j^\top \bx}}.
\end{align*}
Throughout we follow the convention that $\bw_k=0$. Thus the gradient of $Q$ with respect to the $i^{\text{th}}$ gating parameter $\bw_i$ is given by
\begin{align*}
\nabla_{\bw_i} Q(\bw|\bw_t) = \Expect \qth{ \pth{p_{\bw_t}^{(i)}- \frac{e^{\bw_i^\top \bx} }{ 1+\sum_{ j \in [k-1]} e^{\bw_j^\top \bx}}} \cdot \bx  }, \quad i \in [k-1].
\end{align*}
Thus the $(i,j)^{\text{th}}$ block of the negative Hessian $-\nabla^{(2)}_{\bw} Q(\bw|\bw^\ast) \in \reals^{d(k-1) \times d (k-1)}$ is given by
\begin{align}
-\nabla_{\bw_i,\bw_j} Q(\bw|\bw^\ast) = \begin{cases}
\Expect[ p_i(\bx) (1-p_i(\bx)) \cdot \bx \bx^\top], & j=i \\
\Expect[ -p_i(\bx) p_j(\bx) \cdot \bx \bx^\top], & j \neq i
\end{cases}, \label{eq:blockhessian}
\end{align}
where $p_i(\bx)=\frac{e^{\bw_i^\top \bx}}{1+\sum_{j \in [k-1] e^{\bw_j^\top \bx} }}$. It is clear from \prettyref{eq:blockhessian} that $-\nabla^{(2)}_{\bw} Q(\bw|\bw^\ast)$ is positive semi-definite. Since we are interested in the strong convexity of $-Q(\cdot|\bw^\ast)$ which is equivalent to positive definiteness of the negative Hessian, it suffices to show that
\begin{align*}
\lambda \triangleq \inf_{w \in \Omega} \lambda_{\min}\pth{-\nabla^{(2)}_{\bw} Q(\bw|\bw^\ast)} > 0.
\end{align*}
Since the Hessian is continuous with respect to $\bw$ and consequently the minimum eigenvalue of it, there exists a $\bw' \in \Omega$ such that
\begin{align*}
\lambda = \lambda_{\min}\pth{-\nabla^{(2)}_{\bw'} Q(\bw'|\bw^\ast)}=\inf_{\norm{\ba}=1} \ba^\top \pth{-\nabla^{(2)}_{\bw'} Q(\bw'|\bw^\ast)} \ba,
\end{align*}
where $\ba=(\ba_1^\top,\ldots,\ba_{k-1}^\top)^\top \in \reals^{d(k-1)}$. In view of \prettyref{eq:blockhessian}, the above equation can be further simplified to
\begin{align}
\lambda= \inf_{\norm{\ba}=1} \Expect[\ba_{\bx}^\top M_{\bx} \ba_{\bx}],\label{eq:infattain}
\end{align}
where $\ba_{\bx}=(\ba_1^\top \bx,\ldots,\ba_{k-1}^\top \bx)^\top \in \reals^{k-1}$ and $M_{\bx}$ is given by
\begin{align*}
M_{\bx}(i,j) = \begin{cases}
p_i(\bx)(1-p_i(\bx)),  \quad i=j \\
-p_i(\bx)p_j(\bx), \quad i \neq j
\end{cases}
\end{align*}
Let the infimum in \prettyref{eq:infattain} is attained by $\ba^\ast$, \ie $\lambda= \Expect[(\ba_{\bx}^\ast)^\top M_{\bx} \ba_{\bx}^\ast ]$. For each $\bx$, $M_{\bx}$ is strictly diagonally dominant since $|M_{\bx}(i,i)|=p_i(\bx)(1-p_i(\bx))=p_i(\bx)\pth{\sum_{j \neq i, j \in [k]}p_j(\bx)}> p_i(\bx)\pth{\sum_{j \neq i, j \in [k-1]}p_j(\bx)}=\sum_{j \neq i}M(i,j)$. Thus $M_{\bx}$ is positive-definite and $(\ba_{\bx}^\ast)^\top M_{\bx} \ba_{\bx}^\ast >0$ whenever $\ba^\ast_{\bx} \neq 0$. Since $x$ follows a continuous distribution it follows that $\ba^\ast_{\bx} \neq 0$ with probability $1$ and thus $\lambda= \Expect[(\ba_{\bx}^\ast)^\top M_{\bx} \ba_{\bx}^\ast ] >0$.

Now it remains to show that \prettyref{assump:A3} too holds, \ie
\begin{align*}
\norm{\nabla Q(M(\bw)|\bw^\ast) - \nabla Q(M(\bw)|\bw)} \leq \gamma \norm{\bw -\bw^\ast}.
\end{align*}
Note that $\bw=(\bw_1^\top,\ldots,\bw_{k-1}^\top)^\top \in \reals^{d(k-1)}$. We will show that
\begin{align*}
\norm{(\nabla Q(M(\bw)|\bw^\ast))_i - ( \nabla Q(M(\bw)|\bw))_i} \leq \gamma_{\sigma} \norm{\bw -\bw^\ast}, \quad i \in [k-1],
\end{align*}
where $( \nabla Q(M(\bw)|\bw))_i \in \reals^{d}$ refers to the $i^{\text{th}}$ block of the gradient and $\gamma_\sigma \rightarrow 0$. Observe that
\begin{align*}
(\nabla Q(M(\bw)|\bw^\ast))_i - ( \nabla Q(M(\bw)|\bw))_i = \Expect \qth{ (p_{\bw}^{(i)}- p_{\bw^\ast}^{(i)}) \cdot \bx  }
\end{align*}
Let $\Delta=\bw-\bw^\ast$ and correspondingly $\Delta=(\Delta_1^\top,\ldots,\Delta_{k-1}^\top)^\top$ where $\Delta_i=\bw_i-\bw_i^\ast$. Thus it suffices to show that
\begin{align*}
\norm{ \Expect [(p_{\bw}^{(i)}- p_{\bw^\ast}^{(i)}) \cdot \bx ] } \leq \gamma_\sigma \norm{\Delta}.
\end{align*}
Or equivalently,
\begin{align*}
\Expect [(p_{\bw}^{(i)}- p_{\bw^\ast}^{(i)}) \inner{ \bx}{\tilde{\Delta}} ] \leq \gamma_\sigma \norm{\Delta} \|\tilde{\Delta}\|, \quad \forall \tilde{\Delta} \in \reals^d. 
\end{align*}
We consider the case $i=1$. The proof for the other cases is similar. Recall that 
\begin{align*}
p^{(1)}_{\bw} =  \frac{ p_{1}(\bx) \calN(y|g(\ba_1^\top \bx),\sigma^2)}{\sum_{j \in [k] } p_{j}(\bx) \calN(y|g(\ba_j^\top \bx),\sigma^2)}, \quad p_{i}(\bx)= \frac{e^{\bw_i^\top \bx}}{1+\sum_{ j \in [k-1]}e^{\bw_j^\top \bx}}, \quad i \in [k-1].
\end{align*}
For simplicity we define $N_i=\calN(y|g(\ba_1^\top \bx),\sigma^2)$. It is straightforward to verify that
\begin{align*}
\nabla_{\bw_j} p_i(\bx) = \begin{cases}
p_i(\bx)(1-p_i(\bx)) \cdot \bx, & j=i \\
-p_i(\bx) p_j(\bx) \cdot \bx, & j \neq i
\end{cases}
\end{align*}
Thus 
\begin{align*}
\nabla_{\bw_1} (p^{(1)}_{\bw}) &= \nabla_{\bw_1} \pth{\frac{p_1(\bx)N_1}{\sum_{i=1}^N p_i(\bx) N_i}} \\
&=\frac{\pth{\sum_{i=1}^N p_i(\bx) N_i} p_1(\bx)(1-p_1(\bx))N_1 - p_1(\bx)N_1 \pth{-\sum_{j \neq 1} p_j(\bx)p_1(\bx) N_j +p_1(\bx)(1-p_1(\bx))N_1}}{\pth{\sum_{i=1}^N p_i(\bx) N_i}^2}\cdot \bx\\
&=\frac{p_1(\bx) N_1 \pth{\sum_{j \geq 2} p_j(\bx)N_j}}{\pth{\sum_{i=1}^N p_i(\bx) N_i}^2} \cdot \bx \\
& \triangleq R_1(\bx,y,\bw,\sigma) \cdot \bx
\end{align*}
Similarly,
\begin{align*}
\nabla_{\bw_i} (p^{(1)}_{\bw}) & =\frac{p_1(\bx) p_i(\bx) N_1 N_i}{\pth{\sum_{i=1}^N p_i(\bx) N_i}^2} \cdot \bx, \quad i \neq 1, \\
& \triangleq  R_i(\bx,y,\bw,\sigma) \cdot \bx.
\end{align*}
Let $\bw_u \triangleq \bw^\ast+ u \Delta, u \in [0,1]$ and $f(u) \triangleq p_{\bw_u}^{(1)}$. Thus 
\begin{align*}
p_{\bw}^{(1)}-p_{\bw^\ast}^{(1)}=f(1)-f(0)&=\int_{0}^1 f'(u) du \\
&= \int_0^1 \pth{\sum_{i \in [k-1]} \inner{\nabla_{\bw_i} (p_{\bw_u}^{(1)}) }{\Delta_i}  } du  \\
&=\sum_{i \in [k-1]} \int_0^1 R_i(\bx,y,\bw,\sigma) \inner{\bx}{\Delta_i} du .
\end{align*}
So we get
\begin{align*}
\Expect [(p_{\bw}^{(1)}- p_{\bw^\ast}^{(1)}) \inner{ \bx}{\tilde{\Delta}} ] &= \sum_{i \in [k-1]} \int_0^1 \Expect[ R_i(\bx,y,\bw_u,\sigma) \inner{\bx}{\Delta_i} \inner{\bx}{\tilde{\Delta}}] du \\
& \leq \sum_{i \in [k-1]} \int_0^1 \sqrt{\Expect[R_i(\bx,y,\bw_u,\sigma)^2] \Expect[\inner{\bx}{\Delta_i}^2 \inner{\bx}{\tilde{\Delta}}^2]} du\\
&\leq \sum_{i \in [k-1]} \int_0^1 \sqrt{\Expect[R_i(\bx,y,\bw_u,\sigma)^2]} \pth{\sqrt{3}\norm{\Delta_i}\|\tilde{\Delta}\|} du \\
&\leq \sum_{i \in [k-1]} \int_0^1 \sqrt{\Expect[R_i(\bx,y,\bw_u,\sigma)^2]} \pth{\sqrt{3}\norm{\Delta}\|\tilde{\Delta}\|} du \\
&=\underbrace{\pth{\sum_{i \in [k-1]} \int_0^1 \sqrt{\Expect[R_i(\bx,y,\bw_u,\sigma)^2]} du  }}_{\gamma^{(1)}_\sigma}\pth{\sqrt{3}\norm{\Delta}\|\tilde{\Delta}\|}
\end{align*}
Now our goal is to show that $\Expect[R_i(\bx,y,\bw_u,\sigma)^2] \rightarrow 0$ as $\sigma \rightarrow 0$. For $i=1$, we have
\begin{align*}
R_1(\bx,y,\bw_u,\sigma)^2 = \pth{\frac{\sum_{j \geq 2} p_1(\bx)  p_j(\bx)N_1 N_j}{\pth{\sum_{i=1}^N p_i(\bx) N_i}^2}}^2 \leq k \sum_{j \geq 2} \pth{\frac{ p_1(\bx)  p_j(\bx)N_1 N_j}{\pth{\sum_{i=1}^N p_i(\bx) N_i}^2}}^2 \leq k \sum_{j \geq 2} \pth{\frac{ p_1(\bx)  p_j(\bx)N_1 N_j}{(p_1(\bx)N_1 +p_j(\bx)N_j)^2}}^2
\end{align*}
Similarly,
\begin{align*}
R_i(\bx,y,\bw_u,\sigma)^2 \leq \pth{\frac{ p_1(\bx)  p_i(\bx)N_1 N_i}{(p_1(\bx)N_1 +p_i(\bx)N_i )^2}}^2, \quad \forall i \neq 1, i \in [k-1].
\end{align*}
For $\bw=\bw_u$ and $i \neq 1$, we have that
\begin{align*}
\frac{ p_1(\bx)  p_i(\bx)N_1 N_i}{(p_1(\bx)N_1 +p_i(\bx)N_i )^2} &= \frac{e^{\bw_1^\top \bx}e^{\bw_i^\top \bx} e^{- \frac{(y-g(\ba_1^\top \bx))^2}{2\sigma^2}} e^{- \frac{(y-g(\ba_i^\top \bx))^2}{2\sigma^2}}}{\pth{e^{\bw_1^\top \bx}e^{- \frac{(y-g(\ba_1^\top \bx))^2}{2\sigma^2}}+e^{\bw_i^\top \bx}e^{- \frac{(y-g(\ba_i^\top \bx))^2}{2\sigma^2}} }^2} \leq \frac{1}{4}\\
&=\frac{e^{\bw_1^\top \bx}e^{\bw_i^\top \bx} e^{ \frac{(y-g(\ba_1^\top \bx))^2-(y-g(\ba_i^\top \bx))^2}{2\sigma^2}}}{\pth{e^{\bw_1^\top \bx}+e^{\bw_i^\top \bx}  e^{ \frac{(y-g(\ba_1^\top \bx))^2-(y-g(\ba_i^\top \bx))^2}{2\sigma^2}} }^2}\\
& \xrightarrow{\sigma \rightarrow 0} 0.
\end{align*}
Thus, by Dominated Convergence Theorem, $\Expect[R_i(\bx,y,\bw_u,\sigma)^2 ] \rightarrow 0$ for each $u \in [0,1]$. To show that $\int_0^1 \Expect[R_i(\bx,y,\bw_u,\sigma)^2 ] du \rightarrow 0$, we can now follow the same analysis as in the proof of \prettyref{thm:twoexpertsEM} from \prettyref{eq:needagain} on-wards (replacing $\bw$ there with $\bw_1-\bw_i$) which ensures that $\gamma^{(1)}_\sigma$ in our case converges to zero. Similarly for other $i \in [k-1]$, we get that $\gamma^{(i)} \rightarrow 0$. Taking $\gamma_\sigma=\gamma^{(1)}_\sigma+\ldots+\gamma^{(k-1)}_\sigma$ and $\kappa_\sigma=\frac{\gamma_\sigma}{\lambda}$ completes the proof.

\end{proof}
\section{Gradient EM algorithm}
\label{app:twoexpertsgradEM}
In this section, we provide the convergence guarantees for the gradient EM algorithm. For simplicity, we prove the results for $k=2$ and $(\ba_1,\ba_2)=(\ba_1^\ast,\ba_2^\ast)$. Thus we want to learn the gating parameter $\bw^\ast$ in this setting. The results for the general case follow essentially the same proof as that of \prettyref{thm:twoexpertsEM}. In particular, our \prettyref{thm:twoexpertsgradEM} can be viewed as a generalization of \prettyref{lmm:contraction}. Together with \prettyref{lmm:robustness}, extension to general $k$ is straightforward.

Note that in the M-step of the EM algorithm, instead of maximizing $Q(\cdot|\bw_t)$, we can chose an iterate so that it increases the $Q$ value instead of fully maximizing it, \ie 
$Q(\bw_{t+1}|\bw_t) \geq Q(\bw_t|\bw_t)$. Such a procedure is termed as \emph{generalized EM}. \emph{Gradient EM} is an example of generalized EM in which we take an ascent step in the direction of the gradient of $Q(\cdot|\bw_t)$ to produce the next iterate, \ie
\begin{align*}
\bw_{t+1} = \bw_t+ \alpha \nabla Q(\bw_t|\bw_t),
\end{align*}
where $\alpha >0$ is a suitably chosen step size and the gradient is with respect to the first argument. To account for the constrained optimization, we can include a projection step. Mathematically,
\begin{align*}
\bw_{t+1}= G(\bw_t), \quad G(\bw)=\Pi_{\Omega}(\bw+\alpha \nabla Q(\bw|)\bw),
\end{align*}
where $\Pi_\Omega$ refers to the projection operator. Our next result establishes that the iterates of the gradient EM algorithm too converge geometrically for an appropriately chosen step size $\alpha$. 
\begin{theorem}
\label{thm:twoexpertsgradEM}
Suppose that the domain $\Omega=\{\bw\in \reals^d: \|\bw\|_2 \leq 1\}$ and $(\ba_1,\ba_2)=(\ba_1^\ast,\ba_2^\ast)$. Then there exist constants $\alpha_0>0$ and $\sigma_0 >0$ such that for any step size $0<\alpha \leq \alpha_0$ and noise variance $\sigma<\sigma_0$, the gradient EM updates on the gating parameter $\{\bw\}_{t \geq 0}$ converge geometrically to the true parameter $\bw^\ast$, \ie
\begin{align*}
\norm{\bw_t-\bw^\ast} \leq \pth{\rho_\sigma}^t \norm{\bw_0 -\bw^\ast},
\end{align*}
where $\rho_\sigma$ is a dimension-independent constant depending on $g$ and $\sigma$.
\end{theorem}
\begin{remark}\normalfont
The condition $\sigma<\sigma_0$ ensures that the Lipschitz constant $\rho_\sigma$ for the map $G$ is strictly less than $1$. The constant $\alpha_0$ depends only on two universal constants which are nothing but the strong-concavity and the smoothness parameters for the function $Q(\cdot|\bw^\ast)$.
\end{remark}
\begin{proof}
In addition to the assumptions of \prettyref{app:reviewem}, if we can ensure that the map $-Q(\cdot|\bw^\ast)$ is $\mu$-smooth, then the proof follows from Theorem $3$ of \cite{bala17} if we choose $\alpha_0=\frac{2}{\mu+\lambda}$ where $\lambda$ is the strong-convexity parameter of $-Q(\cdot|\bw^\ast)$. The strong-convexity is already established in \prettyref{app:twoexpertsEM}. To find the smoothness parameter, note that
\begin{align*}
- \nabla^{2} Q(\bw|\bw^\ast)& = \Expect \qth{f'(\bw^\top \bx) \cdot \bx \bx^\top}, \\
&= \Expect \qth{f'''(\bw^\top \bx)} \cdot \bw \bw^\top + \Expect[f'(\bw^\top \bx)] \cdot I \\
&=\Expect[f'''(\norm{\bw} Z)] \cdot  \bw \bw^\top + \Expect[f'(\norm{\bw} Z)] \cdot I, \quad Z \sim \calN(0,1) \\
& \preceq \sup_{0 \leq \alpha \leq 1} \min \sth{\Expect[f'(\alpha Z)],\Expect[f'(\alpha Z)] + \alpha^2 \Expect[f'''(\alpha Z)]} \cdot I\\
&= \underbrace{0.25}_{\mu} \cdot  I.
\end{align*}
The contraction parameter is then given by 
\begin{align*}
\rho_\sigma= 1- \frac{2 \lambda+2\gamma_\sigma}{\mu+\lambda}.
\end{align*}
Since $\gamma_\sigma \xrightarrow{\sigma \rightarrow 0} 0$, $\rho_\sigma<1$ whenever $\sigma< \sigma_0$ for a constant $\sigma_0$. 
\end{proof}
\section{Additional experiments}
\label{app:additionalexp}

\subsection{Synthetic data}
\label{app:app_synthetic_data}

\begin{figure*}[h]
\begin{subfigure}{0.33\textwidth}
\includegraphics[width=\linewidth]{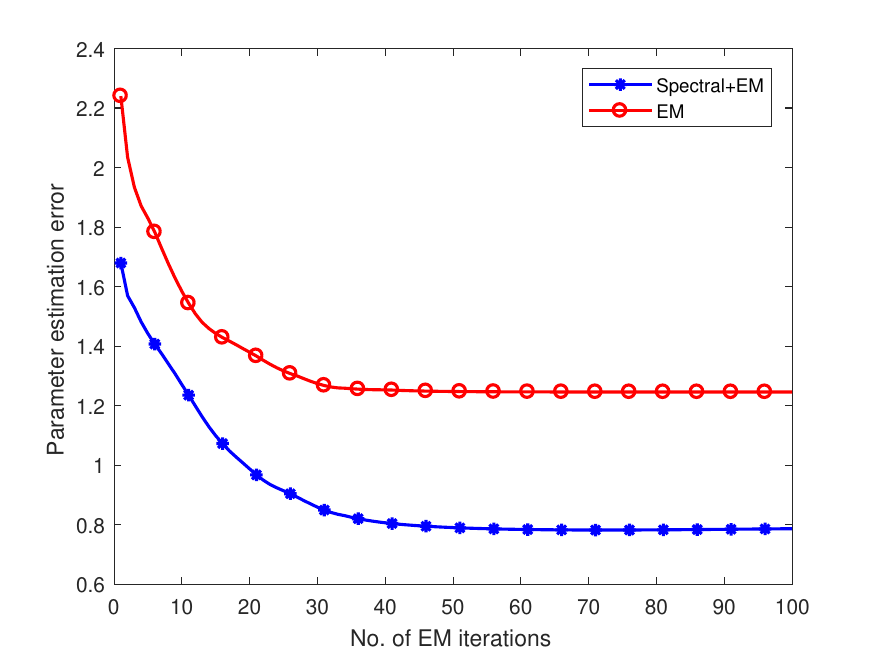}
\caption{} \label{fig:2a}
\end{subfigure}
\begin{subfigure}{0.33\textwidth}
\includegraphics[width=\linewidth]{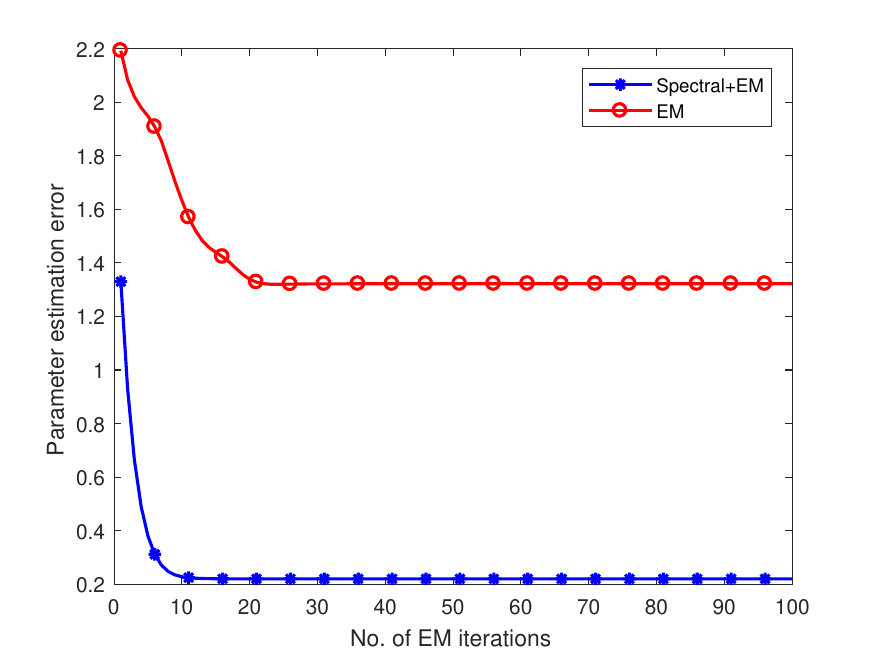}
\caption{} \label{fig:2b}
\end{subfigure}
\begin{subfigure}{0.33\textwidth}
\includegraphics[width=\linewidth]{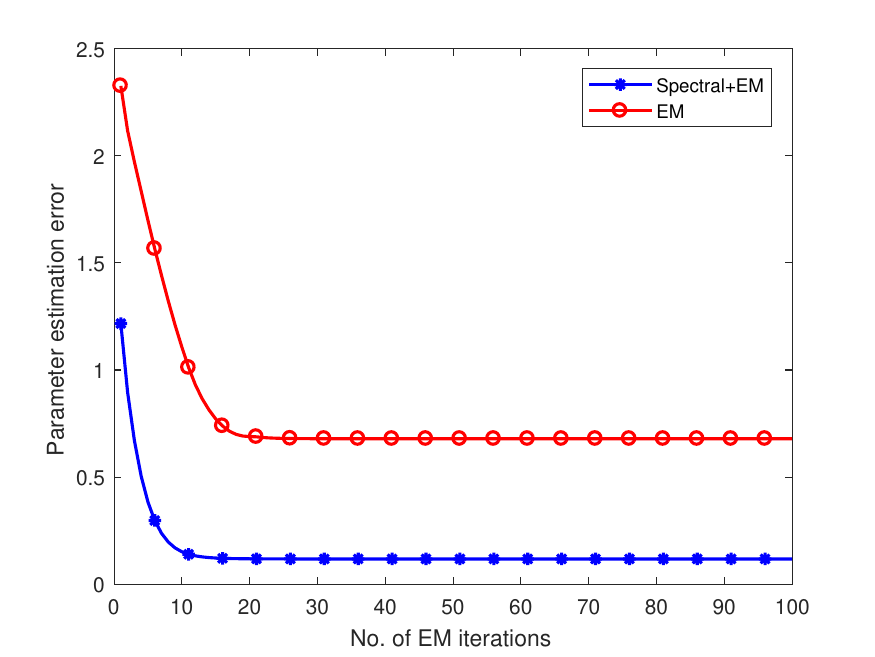}
\caption{} \label{fig:2c}
\end{subfigure}
\caption{Plot of parameter estimation error with varying number of samples($n$): (a) $n=1000$ (b) $n=5000$. (c) $n=10000$.} \label{fig:variablen}
\end{figure*}

In \prettyref{fig:variablen}, we varied the number of samples our data set and fixed the other set
of parameters to $k = 3, d = 5, \sigma = 0.5$.

\begin{figure*}[h]
\begin{subfigure}{0.45\textwidth}
\includegraphics[width=\linewidth]{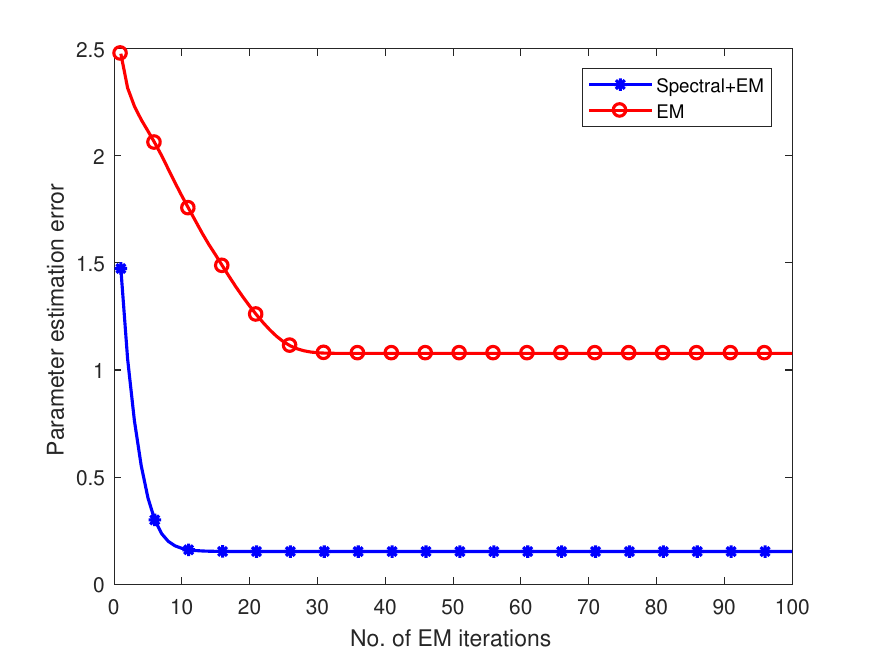}
\caption{} \label{fig:1a}
\end{subfigure}
\hspace*{\fill} 
\begin{subfigure}{0.45\textwidth}
\includegraphics[width=\linewidth]{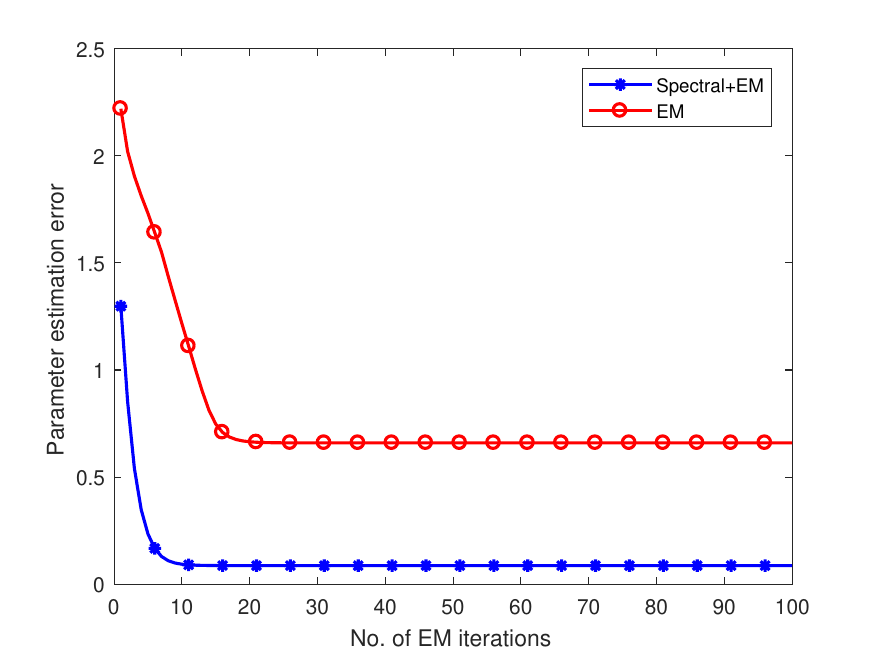}
\caption{} \label{fig:1b}
\end{subfigure}
\caption{Parameter estimation error for the sigmoid and ReLU nonlinearities respectively.} \label{fig:variablenonlinear}
\end{figure*}

In \prettyref{fig:variablenonlinear} we repeated our experiments for the choice of $n=10000, d=5, k=3$ for two different popular choices of non-linearities: sigmoid and ReLU. The same conclusion as in the linear setting holds in this case too with our algorithm outperforming the EM consistently.

\subsection{Real data}
\label{app:app_real_data}
%


\begin{figure*}[t]
\centering
\includegraphics[width=.3\textwidth]{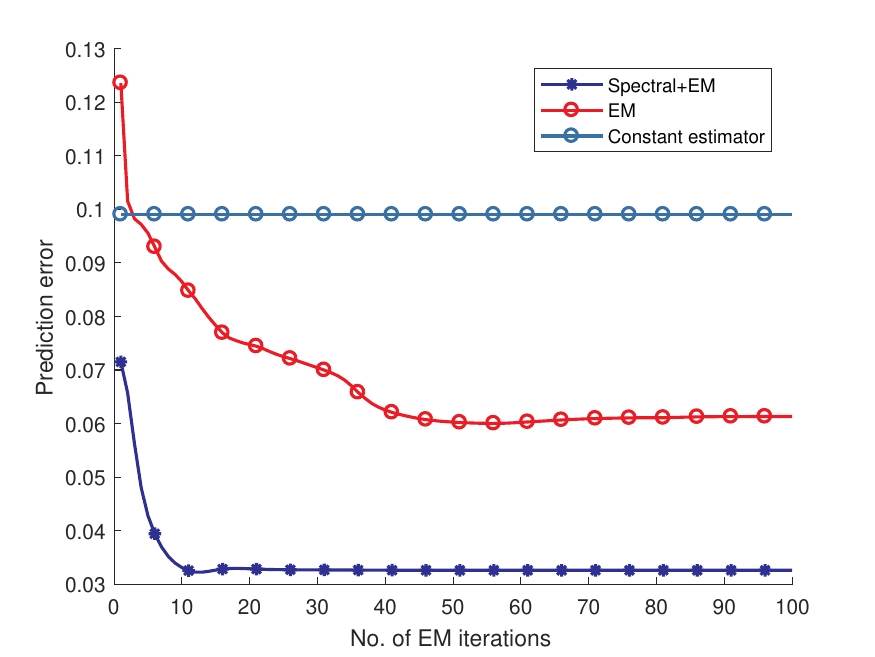}\quad
\includegraphics[width=.3\textwidth]{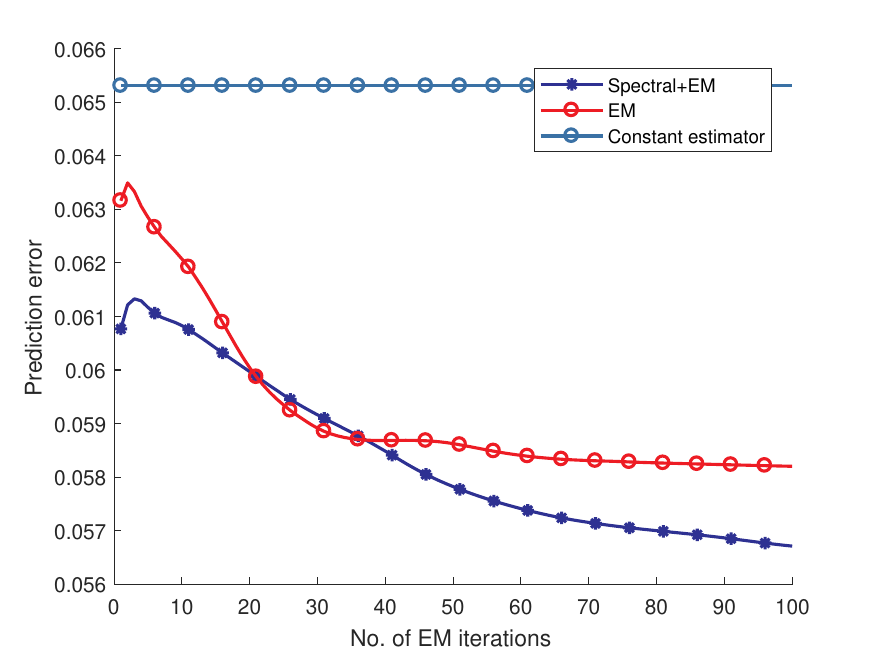}\quad
\includegraphics[width=.3\textwidth]{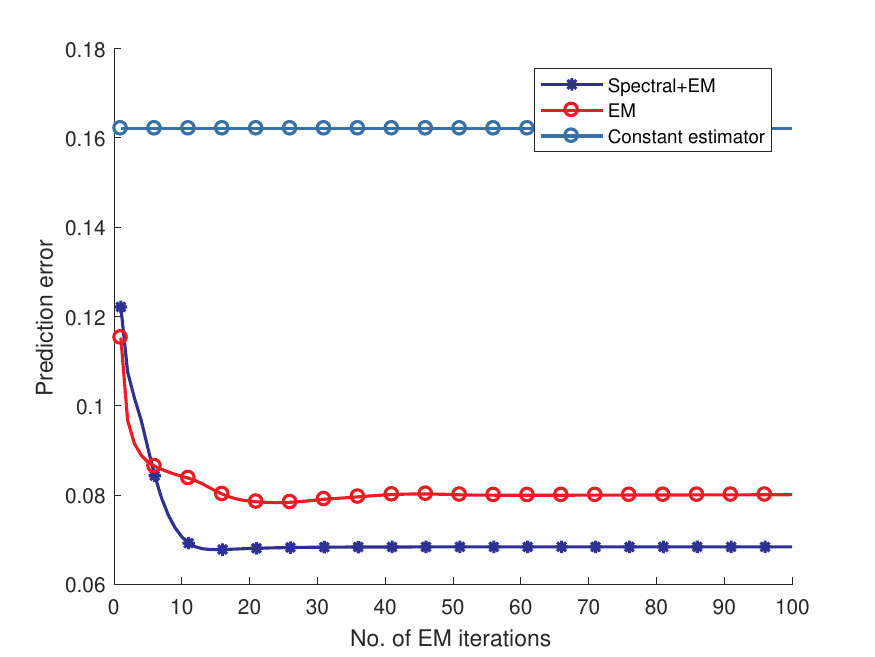}
%
%

\caption{Prediction error for the concrete, stock portfolio and the airfoil data sets respectively.  }
\label{fig:final}
\end{figure*}

For real data experiments, we choose the $3$ standard regression data sets from the UCI Machine Learning Repository: Concrete Compressive Strength Data Set, Stock portfolio performance Data Set, and Airfoil Self-Noise Data Set \cite{concrete,stock,airfoil}. In all the three tasks, the goal is to predict the outcome or the response $y$ for each input $\bx$, which typically contains some task specific attributes. For example, in the concrete compressive strength, the task is to predict the compressive strength of the concrete given its various attributes such as the component of cement, water, age, etc. For this data, the input $\bx \in \reals^8$ corresponds to $8$ different attributes of the concrete and the output $y \in \reals$ corresponds to its concrete strength. Similarly, for the stock portfolio data set the input $\bx \in \reals^6$ contains the weights of several stock-picking concepts such as  weight of the Large S/P concept,  weight of the Small systematic Risk concept, etc,. and the output $y$ is the corresponding excess return. The airfoil data set is obtained from a series of aerodynamic and acoustic tests of two and three-dimensional airfoil blade sections and the goal is predict the scaled sound pressure level (in dB) given the frequency, angle of attack, etc,. For all the tasks, we pre-processed the data by whitening the input and scaling the output to lie in $(-1,1)$. We randomly allotted $75\%$ of the data samples for training and the rest for testing. Our evaluation metric is the prediction error on the test set $(\bx_i,y_i)_{i=1}^n$ defined as 
\begin{align*}
\calE =  \frac{1}{n} \sum_{i=1}^n (\hat{y}_i-y_i)^2,
\end{align*}
where $\hat{y}_i$ corresponds to the predicted output response using the learned parameters. In other words, 
$$\hat{y}=  \sum_{i \in [k]} \frac{e^{\hat{\bw}_i^\top \bx}}{\sum_{j \in [k]}e^{\hat{\bw}_j^\top \bx}} \cdot g(\hat{\ba}_i^\top \bx).
$$
We ran the joint-EM algorithm (with $10$ different trails) on these tasks with various choices for $k \in \{2,\ldots,10\}, \sigma \in \{0.1,0.4,0.8,1 \}, g \in \{\text{linear}, \text{sigmoid},\text{ReLU} \}$ and found the best hyper-parameters to be $(k=3,\sigma=0.1$ and $ g = \text{linear})$, $(k=3,\sigma=0.4,g=\text{sigmoid})$ and $(k=3,\sigma=0.1,g=\text{linear})$ for the three datasets respectively. For this choice of best hyper-parameters found for joint-EM, we ran our algorithm. \prettyref{fig:final} highlights the predictive performance of our algorithm as compared to that of the EM. We also plotted the variance of the test data for reference and to gauge the performance of our algorithm. In all the settings our algorithm is able to obtain a better set of parameters resulting in smaller prediction error.

\end{document}